\documentclass[twoside,11pt]{article}
\usepackage[abbrvbib, preprint]{jmlr2e}

\usepackage{subfig}

\newtheorem{assumption}{Assumption}

\usepackage{amsmath}
\usepackage{lastpage}

\ShortHeadings{Randomly Projected Convex Clustering Model}{Wang, Yuan, Ma, Zeng, and Sun}
\firstpageno{1}

\begin{document}

\title{Randomly Projected Convex Clustering Model: Motivation, Realization, and Cluster Recovery Guarantees}

\author{\name Ziwen Wang \email zwwang@math.cuhk.edu.hk \\
       \addr Department of Mathematics\\
       The Chinese University of Hong Kong\\
       Hong Kong
       \AND
       \name Yancheng Yuan\thanks{Corresponding author.} \email yancheng.yuan@polyu.edu.hk \\
       \addr Department of Applied Mathematics\\
       The Hong Kong Polytechnic University\\
       Hong Kong
       \AND
       \name Jiaming Ma \email 22051002r@connect.polyu.hk \\
       \addr Department of Applied Mathematics\\
       The Hong Kong Polytechnic University\\
       Hong Kong
       \AND
       \name Tieyong Zeng \email zeng@math.cuhk.edu.hk \\
       \addr Department of Mathematics\\
       The Chinese University of Hong Kong\\
       Hong Kong
       \AND
       \name Defeng Sun \email defeng.sun@polyu.edu.hk \\
       \addr Department of Applied Mathematics\\
       The Hong Kong Polytechnic University\\
       Hong Kong
       }

\editor{}

\maketitle

\begin{abstract}
In this paper, we propose a randomly projected convex clustering model for clustering a collection of $n$ high dimensional data points in $\mathbb{R}^d$ with $K$ hidden clusters. Compared to the convex clustering model for clustering original data with dimension $d$, we prove that, under some mild conditions, the perfect recovery of the cluster membership assignments of the convex clustering model, if exists, can be preserved by the randomly projected convex clustering model with embedding dimension $m = O(\epsilon^{-2}\log(n))$, where $0 < \epsilon < 1$ is some given parameter. We further prove that the embedding dimension can be improved to be $O(\epsilon^{-2}\log(K))$, which is independent of the number of data points. Extensive numerical experiment results will be presented in this paper to demonstrate the robustness and superior performance of the randomly projected convex clustering model. The numerical results presented in this paper also demonstrate that the randomly projected convex clustering model can outperform the randomly projected K-means model in practice.
\end{abstract}


\section{Introduction}
Clustering is a fundamental and important problem in data science. Among many others, K-means is arguably the most popular model. It has been widely known that K-means may suffer from the nonconvexity of the model and is very sensitive to the initialization. More critically, K-means requires the number of clusters as a prior, which is not practical in many applications. Recently, researchers have proposed the convex clustering model, which aims to overcome the aforementioned challenges \citep{pelckmans2005convex,hocking2011clusterpath,lindsten2011clustering}.

Given a collection of $n$ data points with $d$ features $A=\{\mathbf{a}_1, \mathbf{a}_2, \ldots, \mathbf{a}_n\} \subseteq \mathbb{R}^{d}$, the general weighted convex clustering model (CCM) solves the following convex optimization problem
\begin{equation}
\tag{CCM}
\label{model: WSON}
\min _{x_1, \dots, x_n \in \mathbb{R}^{d}} ~ \frac{1}{2} \sum_{i=1}^n\left\|\mathbf{x}_i-\mathbf{a}_i\right\|^2+\gamma \sum_{i<j} w_{i j}\left\|\mathbf{x}_i-\mathbf{x}_j\right\|_q,
\end{equation}
where $w_{i j}=w_{j i} \geq 0$ are given weights depending on the input data $A$, $\gamma>0$ is a tuning parameter which controls the strength of the fusion penalty, and $\|\cdot\|_q$ is the vector $q$-norm ($q \geq 1$). In this paper, we focus on the convex clustering model with $q = 2$. We denote $\|\cdot\|$ as the vector $2$-norm. One choice of the weights is setting $w_{ij} = 1$ for all $1 \leq i < j \leq n$, and the resulting model is usually called the convex clustering model with uniform weights. In practice, the following k-nearest neighbors-based weights are popular due to their robustness and computational efficiency:
\begin{equation}
\label{eq: gaussian-weights}
w_{ij} = \left\{
\begin{array}{ll}
 \exp(-\phi\|\mathbf{a}_i - \mathbf{a}_j\|^2), &  \mbox{if $(i, j) \in \mathcal{E}(k)$},\\
  0,   & \mbox{otherwise},
\end{array}
\right.
\end{equation}
here, $\mathcal{E}(k) := \{(i, j) ~\mid~ \mbox{if $\mathbf{a}_i$ (or $\mathbf{a}_j$) is in $\mathbf{a}_j$'s (or $\mathbf{a}_i$'s) k-nearest neighbors}, 1 \leq i \neq j \leq n \}$.

Extensive investigation has been conducted for the convex clustering model in recent years and impressive progress has been achieved from the perspectives of both the recovery properties and efficient numerical algorithms. From the theoretical understanding perspective, some deterministic and statistical cluster recovery guarantees have been established \citep{zhu2014convex, tan2015statistical, panahi2017clustering,Radchenko2017,chiquet2017fast,chi2019recovering,sun2021convex,conve-coclustering,jiang2020recovery,dunlap2022local}. More specifically, under some mild conditions, there exists a nonempty interval of the tuning parameter $\gamma$ such that the convex clustering model can perfectly recover the cluster membership of the data \citep{panahi2017clustering,sun2021convex}. 
From the perspective of optimization algorithms, impressive progress has been achieved in solving the convex clustering model with a large number of data points but with moderate feature dimensions (say with $d \leq 100$ in (\ref{model: WSON})). Along this direction, \citet{chi2015splitting} adopted the alternating direction method of multipliers (ADMM) and proposed an alternating minimization algorithm (AMA). Later, \citet{yuan2018efficient} designed a semismooth Newton based augmented Lagrangian (SSNAL) method that can solve the convex clustering model efficiently with high accuracy. More recently, by taking the advantage of the structured sparsity of the convex clustering model, \citet{yuan2021dimension} proposed dimension reduction techniques (in the sense of the number of data points) called adaptive sieving (AS) and enhanced adaptive sieving (EAS), which further accelerate SSNAL (and other algorithms). Consequently, the existing algorithms can be scalable with respect to the number of data points. However, it is still very challenging to solve the convex clustering model when the dimension of the data features is high (i.e. $d$ is large in (\ref{model: WSON})).

In this paper, we will design a dimension reduction technique for overcoming the computational challenges of the convex clustering model for clustering high dimensional data. Our approach is inspired by the Johnson-Lindenstrauss (JL) lemma \citep{JLLemma84} and the fact that the recovery guarantees of the convex clustering model mainly depend on the pair-wise distances among the data points and centroids. In particular, we will propose a randomly projected (weighted) convex clustering model which clusters the data with a much smaller dimension obtained by applying a random projection mapping to the input data. Among other advantages,
we want to mention that random projection is a computationally efficient approach to obtaining the embedding data.  Importantly, we will prove that the randomly projected convex clustering model will preserve the recovery guarantees of the original convex clustering model. In other words, if there exists a nonempty interval of the parameter $\gamma$ such that the convex clustering model (\ref{model: WSON}) perfectly recovers the cluster memberships of the input data, so will be the randomly projected model in high probability. This is a very interesting and inspiring result since we can obtain the clustering results of the original high dimensional data by solving a more tractable randomly projected convex clustering model with much smaller dimensions. Moreover, we will establish the cluster recovery guarantees for the randomly projected convex clustering model where the embedding dimension can be independent of the number of data points. Extensive numerical experiment results will be presented in this paper to justify the theoretical guarantees and to demonstrate the superior performance and robustness of the proposed model. To further demonstrate the superior performance of the randomly projected convex clustering model, we also compare its performance to the randomly projected K-means model \citep{JL-Kmeans15,k-means-improved}.

We summarize the main contributions of this paper as follows:
\begin{enumerate}
    \item We propose a randomly projected convex clustering model which is much more computationally tractable than the convex clustering model (\ref{model: WSON}).

    \item We establish the recovery guarantees of the randomly projected convex clustering model under mild conditions. We further prove that the embedding dimension can be independent of the number of data points.

    \item We conduct extensive numerical experiments to justify the established theoretical guarantees and demonstrate the superior performance of the proposed randomly projected convex clustering model.
\end{enumerate}

The rest of the paper is organized as follows: In Section \ref{sec: preliminary}, we introduce some concepts and notation and then review some necessary preliminary results of the recovery guarantees of the convex clustering model and the JL lemma. In Section \ref{sec: PCCM}, we will propose a randomly projected convex clustering model and prove its theoretical recovery guarantees. We will then present the numerical results in Section \ref{sec: numerical results}. We will conclude the paper and include some discussion of future research directions in Section \ref{sec: conclusion}.

\section{Preliminaries}
\label{sec: preliminary}
In this section, we first introduce some commonly used notation and then introduce some results about the convex clustering model and the Johnson-Lindenstrauss lemma.

\subsection{Problem Settings}
In this paper, we focus on the following problem setting.

\medskip
\noindent\textbf{General problem setting}: Cluster a collection of $n$ given data points $A=\left\{\mathbf{a}_1, \ldots, \mathbf{a}_n\right\} \subseteq \mathbb{R}^d$ with a hidden clustering partition $\mathcal{V}=$ $\left\{V_1, V_2, \ldots, V_K\right\}$.

\medskip
We define some notation in Table \ref{tab: notation}, which will be commonly used later in this paper.

\begin{table}[tbhp]
\caption{Some commonly used notation. In this table, we assume by default that $1 \leq \alpha \neq \beta \leq K$.}
\label{tab: notation}
{\footnotesize
\begin{center}
\begin{tabular}{c|c}
\hline
 Notation & Definition
\\\hline
$I_{\alpha}$ & $\left\{i ~\mid~ \mathbf{a}_i \in V_\alpha\right\}$ \\
\hline
$n_{\alpha}$ & cardinality of $I_{\alpha}$\\
\hline
$[m]$ for a given integer $m >0$ & $[m] := \left\{1, 2, \dots, m\right\}$ \\
\hline
$\mathbf{a}^{(\alpha)}$ & $\frac{1}{n_\alpha} \sum_{i \in I_\alpha} \mathbf{a}_i$\\
\hline
$\mathbf{a}^{(0)}$ & $\frac{1}{n}\sum_{i=1}^n \mathbf{a}_i$ \\
\hline
$w^{(\alpha, \beta)}$ & $\sum_{i \in I_\alpha} \sum_{j \in I_\beta} w_{i j}$ \\
\hline
$\bar{w}^{(\beta)}$ & $\frac{1}{n_\beta} \sum_{1 \leq l \leq K, l \neq \beta} w^{(\beta, l)}$ \\
\hline
$w_i^{(\beta)}$ ~ $(i \in [n])$ & $\sum_{j \in I_\beta} w_{i j}$\\
\hline
$\mu_{i j}^{(\alpha)}$ ~ $(i,j \in I_{\alpha})$ & $\sum_{\beta=1, \beta \neq \alpha}^K\left|w_i^{(\beta)}-w_j^{(\beta)}\right|$ \\
\hline
$C(n, k) ~ (1 \leq k \leq n)$ & $\frac{n!}{k!(n - k)!}$\\
\hline
\end{tabular}
\end{center}
}
\end{table}
Following the settings in \citep{sun2021convex}, we assume the following assumptions hold throughout this paper.

\begin{assumption}
\label{ass: unique-centroid}
In the general problem setting, the mean vector $\mathbf{a}^{(0)}$ and the centroids $\mathbf{a}^{(1)}, \dots, \mathbf{a}^{(K)}$ are all distinct.
\end{assumption}

\begin{assumption}
\label{ass: weight-assumption}
The specified weights $w_{ij}$ in the model (\ref{model: WSON}) satisfy
\begin{equation}
\label{eq: weights-assump}
\mbox{$w_{ij} > 0$ ~ and ~ $n_{\alpha}w_{ij} > \mu^{(\alpha)}_{ij}$}, \quad \forall i, j \in I_{\alpha}, 1 \leq \alpha \leq K.
\end{equation}
\end{assumption}
A quick comment is that Assumption \ref{ass: weight-assumption} holds automatically for uniform weights. The next definition will be useful for the discussion of the convex clustering model.

\begin{definition}
We say that a map $\psi: \mathbb{R}^d \rightarrow \mathbb{R}^{\bar{d}}$ perfectly recovers $\mathcal{V}$ on the data $A$ if $\psi\left(\mathbf{a}_i\right)=\psi\left(\mathbf{a}_j\right)$ is equivalent to $\mathbf{a}_i$ and $\mathbf{a}_j$ belonging to the same $V_{\alpha}$ for some $1 \leq \alpha \leq K$. We call a partition $\mathcal{W}=\left\{W_1, \ldots, W_L\right\}$ of $A$ a coarsening of $\mathcal{V}$ if there exists a partition $\{\alpha_1, \dots, \alpha_L\}$ of $[K]$ such that $W_l = \bigcup_{i \in \alpha_l} V_i$ for all $1 \leq l \leq L$. We call $\mathcal{W}$ a non-trivial coarsening of $\mathcal{V}$ if $L > 1$.
\end{definition}

\subsection{Recovery guarantees for convex clustering model (\ref{model: WSON})}
In this section, we review the recovery guarantees of the weighted convex clustering model.
\begin{theorem}[\rm{\citep[Theorem 5]{sun2021convex}}]
\label{thm: WSON}
In the general problem setting, denote the optimal solution of the convex clustering model (\ref{model: WSON}) with some given parameter $\gamma \geq 0$ by $\left\{\mathbf{x}_i^*(\gamma)\right\}_{i = 1}^n$ and define the map $\phi_{\gamma}\left(\mathbf{a}_i\right)=\mathbf{x}_i^*(\gamma)$ for $i=1, \ldots, n$.
Define
\begin{equation}
\footnotesize
\label{equ: gamma min &max}
\begin{array}{c}
\gamma_{\min } := \max _{1 \leq \alpha \leq K} \max _{i, j \in I_\alpha}\left\{\frac{\left\|\mathbf{a}_i-\mathbf{a}_j\right\|}{n_\alpha w_{i j}-\mu_{i j}^{(\alpha)}}\right\}, \;
\gamma_{\max } := \min _{1 \leq \alpha<\beta \leq K}\left\{\frac{\left\|\mathbf{a}^{(\alpha)}-\mathbf{a}^{(\beta)}\right\|}
{\bar{w}^{(\alpha)}+\bar{w}^{(\beta)}}\right\},\\
\gamma_{\max2} := \max _{1 \leq \alpha \leq K} \frac{\left\|\bar{\mathbf{a}}-\mathbf{a}^{(\alpha)}\right\|}{\bar{w}^{(\alpha)}}, \; r := \frac{\gamma_{\max}}{\gamma_{\min}}, \; r_2 := \frac{\gamma_{\max2}}{\gamma_{\min}}.
\end{array}
\end{equation}
Under Assumption \ref{ass: unique-centroid} and Assumption \ref{ass: weight-assumption}, we have
\begin{enumerate}
    \item If $r > 1$ and $\gamma \in\left[\gamma_{\min }, \gamma_{\max }\right)$, then the map $\phi_{\gamma}$ perfectly recovers $\mathcal{V}$.
\item If $r_2 > 1$ and $\gamma \in [\gamma_{\min }, \gamma_{\max2})$, then the map $\phi_{\gamma}$ recovers a non-trivial coarsening of $\mathcal{V}$.
\end{enumerate}
\end{theorem}

\subsection{Johnson-Lindenstrauss Lemma and the Random Projection}
In this section, we introduce the Johnson-Lindenstrauss (JL) lemma, which is a key tool for this paper. Consider a collection of high-dimensional data points $X = \left\{\mathbf{x}_1, \dots, \mathbf{x}_n\right\} \subseteq \mathbb{R}^d$, the JL lemma shows the existence of a mapping $f: X \rightarrow \mathbb{R}^{m}$ such that for all points $\mathbf{x}_i \neq \mathbf{x}_j \in X$, $\|\mathbf{x}_i - \mathbf{x}_j\|^2$ are approximately maintained in a $m$ dimensional space within a distortion tolerance $\epsilon\in(0,1)$. More surprisingly, the required embedding dimension $m=O(\epsilon^{-2}\log(n))$ is independent of $d$.
\begin{theorem}[JL lemma \rm{\cite[Lemma 1]{JLLemma84}}]
\label{thm: JL}
For any given collection of $n$ data points $X = \left\{\mathbf{x}_1, \dots, \mathbf{x}_n \right\} \subseteq \mathbb{R}^d$ and any $\epsilon \in(0,1)$, there exists an $\epsilon$- isometry embedding $f: \mathbb{R}^d \rightarrow \mathbb{R}^m$ with $m=O\left(\min\{d, \epsilon^{-2} \log n \}\right)$. In other words, $\forall ~ \mathbf{x}_i, \mathbf{x}_j \in X$,
\begin{equation}\label{equ: JLT}
(1 - \epsilon)\|\mathbf{x}_i - \mathbf{x}_j\|^2 \leq \|f(\mathbf{x}_i)-f(\mathbf{x}_j)\|^2 \leq (1 + \epsilon)\|\mathbf{x}_i - \mathbf{x}_j\|^2.
\end{equation}
\end{theorem}
We also call a mapping $f$ satisfies (\ref{equ: JLT}) an $\epsilon$-JL Transform (or $\epsilon$-JLT in short) on $X$. The mapping $f$ can be found in randomized polynomial time \citep{DG03}. Moreover, if the mapping $f$ must be linear, then $m = \Omega\left(\min\{d, \epsilon^{-2} \log n \}\right)$ is optimal \citep{JNelson16}. The following Distributional Johnson-Lindenstrauss (DJL) lemma is useful.
\begin{theorem}[DJL lemma]
\label{thm: DJL}
For any $\epsilon\in (0,1), \delta \in(0,1 / 2)$ and integer $d>1$, there exists a distribution $\mathcal{D}_{\epsilon, \delta}$ over matrices $\Pi \in \mathbb{R}^{m \times d}$ for $m=O\left(\epsilon^{-2} \log (1 / \delta)\right)$ such that for any $z \in \mathbb{R}^d$ with $\|z\|=1$,
\begin{equation}\label{equ: DJL}
\mathbb{P}_{\Pi \sim \mathcal{D}_{\epsilon, \delta}}\left[\left|\|\Pi z\|^2-1\right|>\epsilon\right]<\delta.
\end{equation}
\end{theorem}
We call a distribution $\mathcal{D}_{\epsilon, \delta}$ that satisfies (\ref{equ: DJL}) a DJL distribution.

For later convenience, we include the following proposition, which is a direct consequence of Theorem \ref{thm: DJL} and the union bound in the probability theory.
\begin{proposition}[Random projection for multiple sets]
\label{prop: union bound DJL}
Assume that there are $l$ non-empty collections of data points $X_1, \dots, X_l$ in $\mathbb{R}^d$ with $|X_{j}|=n_j$ $(1 \leq j \leq l)$. Denote $X = \bigcup_{j = 1}^l X_j$. Given any $0<\epsilon <1$ and $0<\delta<\frac{1}{\sum_{j=1}^{l}{n_j}}$, and let $D_{\epsilon,\delta}$ be a DJL distribution over $\mathbb{R}^{m\times d}$ with $m=O(\epsilon^{-2}\log(1/\delta))$. We have
\begin{equation}
\label{DJLT_Aj_prob}
\mathbb{P}_{\Pi \sim \mathcal{D}_{\epsilon, \delta}}\left[(1 - \epsilon)\|\mathbf{x}\|^2 \leq \|\Pi \mathbf{x}\|^2 \leq (1 + \epsilon)\|\mathbf{x}\|^2, ~ \forall \mathbf{x} \in X \right] \geq 1 - (\sum_{j = 1}^l n_j)\delta > 0.
\end{equation}
Thus, there exists a matrix $\Pi\in\mathbb{R}^{m\times d}$ such that
$$
(1 - \epsilon)\|\mathbf{x}\|^2 \leq \|\Pi \mathbf{x}\|^2 \leq (1 + \epsilon)\|\mathbf{x}\|^2, ~ \forall \mathbf{x} \in X.
$$
\end{proposition}

\section{A Randomly Projected Convex Clustering Model}
\label{sec: PCCM}
The convex clustering model (\ref{model: WSON}) has promising recovery guarantees, but solving the model can be computationally challenging, especially when the feature dimension $d$ is high. In this section, we will propose a randomly projected convex clustering model of (\ref{model: WSON}) with much smaller feature dimensions. We will prove that the recovery guarantees will be preserved with a high probability for the random projected convex clustering model. More specifically, for the given collection of data points $A = \left\{\mathbf{a}_1, \dots, \mathbf{a}_n \right\} \subseteq \mathbb{R}^d$ considered in the general problem setting and a given $\epsilon \in (0, 1)$, we will construct an $\epsilon$-isometry mapping $f: \mathbb{R}^d \rightarrow \mathbb{R}^m$ with $m = O\left(\min\left\{d, \log(n)/\epsilon^2\right\}\right)$, where $m$ can be much smaller than $d$. We solve the following projected convex clustering model
\begin{equation}
\tag{RPCCM}
\label{model: PSON}
\min_{\hat{X} \in \mathbb{R}^{m \times n}} ~ \frac{1}{2} \sum_{i=1}^n\left\|\hat{\mathbf{x}}_i- f(\mathbf{a}_i)\right\|^2+\gamma \sum_{i<j} w_{i j}\left\|\hat{\mathbf{x}}_i-\hat{\mathbf{x}}_j\right\|.
\end{equation}
In this paper, we will choose $f$ as a random projection matrix motivated by the DJL lemma and call the corresponding model (\ref{model: PSON}) a randomly projected convex clustering model.

\subsection{An $\epsilon$-isometry Mapping for the Convex Clustering Model} A key observation is that the recovery guarantees of the convex clustering (e.g. Theorem \ref{thm: WSON}) mainly depend on the distances between data points within the same cluster and the distance between the centroids of different clusters. Thus, the recovery guarantees of the convex clustering model (\ref{model: WSON}) can be inherited by the model (\ref{model: PSON}) if we can construct an $\epsilon$-isometry mapping for some small enough $\epsilon > 0$ for the data points $A$ and the corresponding centroids. The next proposition shows the existence of a desired $\epsilon$-isometry mapping for the convex clustering model.

\begin{proposition}
\label{prop: DJL_C_X}
For the general problem setting, define $X_{\alpha} := \left\{\mathbf{a}_i - \mathbf{a}_j ~|~ i, j \in I_{\alpha}, i < j  \right\}$ $(1 \leq \alpha \leq K)$, and $X_c := \left\{ \mathbf{a}^{(\alpha)} - \mathbf{a}^{(\beta)} ~|~ 0 \leq \alpha < \beta \leq K \right\}$. Denote $N_1 = \sum_{\alpha = 1}^K |X_{\alpha}| = \sum_{\alpha = 1}^K C(n_{\alpha}, 2) < C(n, 2)$, and $N_2 = |X_c| = C(K + 1, 2)$. For any $0<\epsilon <1$, let $\delta=\frac{1}{(N_1 + N_2)^p}$, where $p > 1$, and let $D_{\epsilon,\delta}$ be a DJL distribution over $\mathbb{R}^{m\times d}$, where $m=O(\epsilon^{-2}\log(1/\delta))=O(p\epsilon^{-2}\log(N_1 + N_2))$.
Then for any $\Pi\in\mathbb{R}^{m\times d}$ randomly drawn from $D_{\epsilon,\delta}$, with probability at least $1 - \frac{1}{(N_1 + N_2)^{p-1}}$ that
\begin{subequations}\label{equ: preserve_A&C}
\begin{align}
&(1 - \epsilon)\|\mathbf{a}_i-\mathbf{a}_j\|^{2} \leq  \|\Pi (\mathbf{a}_i-\mathbf{a}_j)\|^2\
\leq(1 + \epsilon)\|\mathbf{a}_i-\mathbf{a}_j\|^{2},  ~ \mathbf{a}_i , \mathbf{a}_j \in V_{\alpha}, 0 \leq \alpha \leq K, \label{subequ: preserve_A}
\\
&(1 - \epsilon) \|\mathbf{a}^{(\alpha)}-\mathbf{a}^{(\beta)}\|^{2} \leq  \|\Pi (\mathbf{a}^{(\alpha)}-\mathbf{a}^{(\beta)})\|^2 \leq (1 + \epsilon)\|\mathbf{a}^{(\alpha)} - \mathbf{a}^{(\beta)}\|,  ~ 1 \leq \alpha \neq \beta \leq K. \label{subequ: preserve_C}
\end{align}
\end{subequations}
If $N_2 \leq n/2$, then it is enough to take $\delta = \frac{2}{n^{p+1}}$ and $m = O(\epsilon^{-2}(p+1)\log(n))$, and the inequalities (\ref{equ: preserve_A&C}) hold with probability at least $1 - \frac{1}{n^{p - 1}}$.
\end{proposition}

The above proposition can be proved as a consequence of Proposition \ref{prop: union bound DJL}. In practice, only the pair-wise distance between the input data points can be checked after a projection matrix $\Pi$ is randomly sampled (which covers the condition (\ref{subequ: preserve_A})). But an insight is $K$ should be much much smaller than $n$ (which is the reason for us to do clustering). The next corollary shows that the condition (\ref{subequ: preserve_C}) can be satisfied in much higher probability if (\ref{subequ: preserve_A}) holds.
\begin{proposition}
\label{corollary: DJL_C_X_check}
Let $\Pi \in \mathbb{R}^{m \times d}$ be a projection matrix sampled from a DJL distribution $D_{\epsilon, \delta}$ with $m = O(\epsilon^{-2}\log(1/\delta))$ and $\delta = \frac{1}{(N_1 + N_2)^p}$. Let $E_1$ be the event that $\Pi$ satisfies (\ref{subequ: preserve_A}) and $E_2$ be the event that $\Pi$ satisfies (\ref{subequ: preserve_C}), respectively.  Then, the conditional probability $\mathbb{P}\left[E_2 ~|~ E_1 \right]$ satisfies
\begin{equation}
\label{eq: condition-bound-DJL}
\mathbb{P}\left[E_2 ~|~ E_1 \right] \geq 1 - \frac{N_2}{(N_1 + N_2)^p - N_1}.
\end{equation}
If we further assume that $N_2 \leq n/2$ and $\delta = \frac{2}{n^{p+1}}$, then
\begin{equation}
\label{eq: condition-bound-DJL-sp}
\mathbb{P}\left[E_2 ~|~ E_1\right] \geq 1 - \frac{1}{n^p - n + 1}.
\end{equation}
\end{proposition}
\begin{proof}
Direct calculation gives that
\[
\begin{array}{lcl}
\mathbb{P}\left[E_2 ~|~ E_1\right] &=& 1 - \mathbb{P}\left[E_2^c ~|~ E_1\right]\\[5pt]
&=& 1 - \frac{\mathbb{P}\left[E_1 \bigcap E_2^c \right]}{\mathbb{P}\left[E_1\right]}\\[5pt]
&\geq& 1 - \frac{\mathbb{P}\left[E_2^c\right]}{\mathbb{P}\left[E_1\right]}\\[5pt]
&\geq& 1 - \frac{N_2\delta}{1 - N_1\delta}\\[5pt]
&=& 1 - \frac{N_2}{(N_1 + N_2)^p - N_1}.
\end{array}
\]
The inequality (\ref{eq: condition-bound-DJL-sp}) can be proved similarly.
\end{proof}

\begin{remark}
The DJL distribution plays a role in the construction of the $\epsilon$-isometry mapping. Indeed, a vast amount of variants of the DJL lemma have been explored by designing the structure of DJL distributions, including the subgaussians \citep{indyk1998approximate,Achlioptas03,Matou08}, the Fast JL Transform \citep{Ailon2009fast,Ailon2009fast2,Ailon2013almost}, and the Sparse JL Transform \citep{dasgupta2010sparse,kane2010derandomized,kane2014sparser,cohen2018simple}. The particular choice of the DJL distribution is beyond the concern of this paper. In this paper, we will follow  \citep{Matou08} and consider the linear random projection matrix $\Pi=\frac{1}{\sqrt{m}}R \in \mathbb{R}^{m \times d}$, where $R_{ij}$ are independent random variables with zero mean and a uniform subgaussian tail.

\end{remark}

\subsection{Cluster Recovery Guarantees of the Randomly Projected Convex Clustering Model for the General Problem Setting} Next, we will establish the cluster recovery guarantees of the randomly projected convex clustering model (\ref{model: PSON}) for the general problem setting. For later convenience, we introduce some useful notation.

\begin{definition}
In the general problem setting, we consider the randomly projected convex clustering model (\ref{model: PSON}) with some specified weights $w_{ij} = w_{ji} \geq 0 ~ (1 \leq i \neq j \leq n)$ and a randomly sampled projection matrix $\Pi\in\mathbb{R}^{m\times d}$ (for some $m \geq 1$). Without explicitly mentioning the dependence on $\Pi$, we define
\begin{equation}
\label{equ: gamma_hat min &max}
\begin{array}{c}
\hat{\gamma}_{\min} := \max _{1 \leq \alpha \leq K} \max _{i, j \in I_\alpha}\left\{\frac{\left\|\Pi(\mathbf{a}_i-\mathbf{a}_j)\right\|}{n_\alpha w_{i j}-\mu_{i j}^{(\alpha)}}\right\}, \quad
\hat{\gamma}_{\max } :=\min _{1 \leq \alpha<\beta \leq K}\left\{\frac{\left\|\Pi(\mathbf{a}^{(\alpha)}-\mathbf{a}^{(\beta)})\right\|}
{\bar{w}^{(\alpha)}+\bar{w}^{(\beta)}}\right\},\\[5pt]
\hat{\gamma}_{\max2} :=\max _{1 \leq \alpha \leq K} \frac{\left\|\Pi(\mathbf{a}^{(0)}-\mathbf{a}^{(\alpha)})\right\|}{\bar{w}^{(\alpha)}}.
\end{array}
\end{equation}
\end{definition}

The next theorem shows that the recovery properties of the original convex clustering model can be preserved by the randomly projected convex clustering model in high probability. For convenience, we assume the following assumption holds for the rest of this paper.

\begin{assumption}
\label{ass: K-not-big}
The inequality $n > K(K+1)$ holds, where $n$ is the number of data points and $K$ is the number of hidden clusters.
\end{assumption}
The above assumption is mild and it is consistent with the purpose of the clustering task.

\begin{theorem}
\label{thm_PSON2}
Consider the general problem setting and the models (\ref{model: WSON}) and (\ref{model: PSON}) with the same specified weights $w_{ij}$. For any $0<\epsilon <1$, let $\delta=\frac{2}{n^p}$ with some $p>2$, and let $D_{\epsilon,\delta}$ be a DJL distribution over $\mathbb{R}^{m\times d}$ with $m=O(p\epsilon^{-2}\log(n))$. Here and below in this theorem, the notation $O(\cdot)$ depends on the same absolute constant. Without loss of generality, we assume that $m < d$ (or equivalently, we can assume $\sqrt{\frac{O(p\log(n))}{d}}<1$ and $\epsilon \in (\sqrt{\frac{O(p\log(n))}{d}},1)$ ). Define
\begin{equation}
\epsilon_{\min}=\sqrt{\frac{O(p\log(n))}{d}}, ~ \epsilon_{\sup}=\frac{r^{2}-1}{r^{2}+1}, ~ and ~\epsilon_{\sup2}=\frac{r_{2}^{2}-1}{r_{2}^{2}+1},
\end{equation}
where $r$ and $r_2$ are the constants defined in (\ref{equ: gamma min &max}). Let $\Pi\in\mathbb{R}^{m\times d}$ be a random projection matrix drawn from $D_{\epsilon,\delta}$.  Denote the optimal solution of the model (\ref{model: PSON}) with $\Pi$ and $\gamma \geq 0$ by $\left\{\hat{\mathbf{x}}_i^*(\gamma)\right\}_{i=1}^n$ and define the map $\hat{\phi}_{\gamma}\left(\mathbf{a}_i\right)=\hat{\mathbf{x}}_i^*$. Then, we have
\begin{enumerate}
\item If $r>\sqrt{\frac{1+\epsilon_{\min}}{1-\epsilon_{\min}}}$, then $\epsilon_{\min} < \epsilon_{\sup}$. For any $\epsilon \in [\epsilon_{\min}, \epsilon_{\sup})$,
and $\hat{\gamma} \in\left[\sqrt{1+\epsilon}\gamma_{\min }, \sqrt{1-\epsilon}\gamma_{\max }\right)$, with probability over $1-\frac{1}{n^{p-2}}$, the map $\hat{\phi}_{\hat{\gamma}}$ perfectly recovers $\mathcal{V}$.

\item If $r_{2}>\sqrt{\frac{1+\epsilon_{\min}}{1-\epsilon_{\min}}}$, then $\epsilon_{\min} < \epsilon_{\sup2}$. For any $\epsilon \in [\epsilon_{\min}, \epsilon_{\sup2})$,
and $\hat{\gamma} \in\left[\sqrt{1+\epsilon}\gamma_{\min }, \sqrt{1-\epsilon}\gamma_{\max2}\right)$, with  probability over $1-\frac{1}{n^{p-2}}$,
the map $\hat{\phi}_{\hat{\gamma}}$ recovers a non-trivial coarsening of $\mathcal{V}$.
\end{enumerate}
\end{theorem}

\begin{proof}
It directly follows Proposition \ref{prop: DJL_C_X} that, with probability over $1 - \frac{1}{n^{p-2}}$, the following statements hold:
\begin{itemize}
    \item[(i)] The centroids $\left\{ \Pi\mathbf{a}^{(0)}, \dots, \Pi\mathbf{a}^{(K)}\right\}$ of the embedded data are distinct.
    \item[(ii)] The parameters $\hat{\gamma}_{\min}$, $\hat{\gamma}_{\max}$, and $\hat{\gamma}_{\max2}$ defined in (\ref{equ: gamma_hat min &max}) satisfy the following inequalities:
    \[
    \begin{array}{c}
    \sqrt{1-\epsilon}\gamma_{\min } \leq \hat{\gamma}_{\min } \leq \sqrt{1+\epsilon}\gamma_{\min }, \quad \sqrt{1-\epsilon}\gamma_{\max } \leq \hat{\gamma}_{\max } \leq \sqrt{1+\epsilon}\gamma_{\max }, \\[5pt] \sqrt{1-\epsilon}\gamma_{\max2 } \leq \hat{\gamma}_{\max2} \leq \sqrt{1+\epsilon}\gamma_{\max2}.
    \end{array}
    \]
\end{itemize}
The above implies that
\begin{equation}
\label{equ: hat gamma min&max and gamma min&max}
\begin{array}{c}
\left[\sqrt{1+\epsilon}\gamma_{\min}, \sqrt{1-\epsilon}\gamma_{\max}\right)\subseteq\left[\hat{\gamma}_{\min},\hat{\gamma}_{\max}\right),\\
\left[\sqrt{1+\epsilon}\gamma_{\min}, \sqrt{1-\epsilon}\gamma_{\max2}\right)\subseteq\left[\hat{\gamma}_{\min},\hat{\gamma}_{\max2}\right).
\end{array}
\end{equation}

Now, we prove the first part of the theorem. We claim here that it is sufficient to show: if $r>\sqrt{\frac{1+\epsilon_{\min}}{1-\epsilon_{\min}}}$, then $\epsilon_{\min} < \epsilon_{\sup}$, and for any $\epsilon \in [\epsilon_{\min}, \epsilon_{\sup})$, $\left[\sqrt{1+\epsilon}\gamma_{\min}, \sqrt{1-\epsilon}\gamma_{\max}\right)$ is nonempty.
In fact, if $\left[\sqrt{1+\epsilon}\gamma_{\min}, \sqrt{1-\epsilon}\gamma_{\max}\right)$ is nonempty, then by the first inclusion of (\ref{equ: hat gamma min&max and gamma min&max}), $\left[\hat{\gamma}_{\min},\hat{\gamma}_{\max}\right)$ is nonempty. Applying Theorem \ref{thm: WSON} to the embbedded data $\Pi A$ implies that for any
$\hat{\gamma} \in \left[\hat{\gamma}_{\min},\hat{\gamma}_{\max}\right)$, the map $\hat{\phi}_{\hat{\gamma}}$ perfectly recovers $\mathcal{V}$.

On the one hand, we have
\[
\begin{array}{lcl}
r > \sqrt{\frac{1 + \epsilon_{\min}}{1 - \epsilon_{\min}}} &\implies& (1 - \epsilon_{\min}) r^2 > 1 + \epsilon_{\min} \\
& \implies & (r^2 - 1) > \epsilon_{\min}(r^2 + 1)\\
& \implies & \epsilon_{\min} < \epsilon_{\sup}.
\end{array}
\]
This implies that the interval $[\epsilon_{\min}, \epsilon_{\sup})$ is nonempty. On the other hand, we have
\[
\begin{array}{lcl}
 \epsilon < \epsilon_{\sup} & \implies & \epsilon < \frac{r^2 - 1}{r^2 + 1}\\
 & \implies & \frac{1 + \epsilon}{1 - \epsilon} < r^2 \\
 & \implies & \frac{1 + \epsilon}{1 - \epsilon} < \frac{\gamma_{\max}^2}{\gamma_{\min}^2}\\
 & \implies & \sqrt{1 + \epsilon} \gamma_{\min} < \sqrt{1 - \epsilon}\gamma_{\max}.
\end{array}
\]
Thus, we have proved the first part of the theorem. The second part of the theorem can be proved in a similar way.
\end{proof}

We can obtain an $\epsilon$-isometry mapping in randomized polynomial time \citep{DG03} and we can check the $\epsilon$-isometry of the mapping on the data $A$ (but not for the centroids). The following corollary is useful. The proof of the corollary follows directly from Theorem \ref{thm_PSON2} and Proposition \ref{corollary: DJL_C_X_check}.
\begin{corollary}
Let $\Pi \in \mathbb{R}^{m \times d}$ be a random projection matrix drawn from $D_{\epsilon, \delta}$ as in Theorem \ref{thm_PSON2}. If we further assume that $\Pi$ satisfies (\ref{subequ: preserve_A}), then, the statements of Theorem \ref{thm_PSON2} hold with probability at least $1 - \frac{1}{n^{p-1} - n + 1}$ under the same assumptions.
\end{corollary}

\begin{remark}
We want to make some remarks on the obtained recovery guarantees of the model (\ref{model: PSON}).
\begin{enumerate}
    \item For convenience, we assumed $K(K+1) \leq n$. But we can also easily obtain recovery guarantees regarding $N_1$ and $N_2$ as defined in Proposition \ref{prop: DJL_C_X}.
    \item The embedding dimension is $m=O\left(p\epsilon^{-2} \log n\right)$, which only depends on $\epsilon$, $n$, and $p$, but it is independent of the data dimension $d$. Also, the embedding dimension grows very slowly with respect to $n$ (in $O(\log(n))$).
    \item We derives the lower bound $\epsilon_{\min}$ and the upper bound $\epsilon_{\sup}$ of $\epsilon$ for perfect recovery of the model (\ref{model: PSON}). In particular, the lower bound $\epsilon_{\min}$ can be very small for high dimensional data. The upper bound $\epsilon_{\sup}$ depends on the ratio of $\gamma_{\max}$ and $\gamma_{\min}$, and it is independent of the scale of the data.
    \item We want to mention that, the weights used in the model (\ref{model: WSON}) and the model (\ref{model: PSON}) are the same.

    \item The dimension reduction based on the JL lemma has been also investigated for the K-means model \citep{JL-Kmeans15}. However, for the K-means model, only the cost (the optimal objective function value of the K-means model) can be preserved up to a tolerance $\epsilon > 0$. Here, we proved that the perfect recovery guarantee of the convex clustering model can be inherited. A comparison of the empirical performance between the randomly projected K-means model and the randomly projected convex clustering model can be found later in the numerical experiments.
\end{enumerate}
\end{remark}

The embedding dimension $m$ in Theorem \ref{thm_PSON2} depends on $n$ of the order $O(\log(n))$. Next, we will further improves it from $O(\log(n))$ to $O(\log(K))$. The key insights come from the estimate of the spectral norm of the random matrices. The following lemma is useful, which is a direct consequence of Theorem 4.6.1 and Lemma 3.4.2 in \citep{vershynin2018high}.

\begin{lemma}[Two-sided bound on sub-gaussian matrices]
\label{thm: extreme singular subgaussian}
Let $\Pi=\frac{1}{\sqrt{m}}R\in\mathbb{R}^{m\times d}$ ($m\leq d$), where $R_{ij}$ are independent random variables with $\mathbb{E}\left[R_{i j}\right]=0$,$\operatorname{Var}\left[R_{i j}\right]=1$ and the subgaussian norm $\kappa := \|R_{ij}\|_{\psi_2} := \inf \{s > 0: \mathbb{E}[\exp(R_{ij}^2 / s^2)] \leq 2\}$. Let $s_{1}(\Pi)$ be the largest singular value of $\Pi$, and let $s_{m}(\Pi)$ be the smallest singular value of $\Pi$. Then for any $t \geq 0$, we have

\begin{equation}
\label{equ: subgaussian singular value estimation}
\underbar{S}(m,d,t)\leq s_{m}(\Pi) \leq s_{1}(\Pi)\leq\bar{S}(m,d,t)
\end{equation}
with probability at least $1-2 \exp \left(-t^2\right)$. Here, $C_{\kappa}^2>0$ is a constant that only depends on $\kappa$, and
$\bar{S}(m,d,t)=\frac{\sqrt{d} + C_{\kappa}^{2}t}{\sqrt{m}} + C_{\kappa}^{2}$, and $\underbar{S}(m,d,t)=\frac{\sqrt{d} - C_{\kappa}^{2}t}{\sqrt{m}} - C_{\kappa}^{2}$.
\end{lemma}

The next theorem shows that the embedding dimension can be independent of the number of data points $n$.

\begin{theorem}
\label{thm_PSON3}
Consider the general problem setting and the models (\ref{model: WSON}) and (\ref{model: PSON}) with the same specified weights $w_{ij}$. For any $0<\epsilon <1$, let $\Pi=\frac{1}{\sqrt{m}}R\in\mathbb{R}^{m\times d}$ with $m=O(p\epsilon^{-2}\log(K))$, where $R_{ij}$ are independent random variables with $\mathbb{E}\left[R_{i j}\right]=0$,$\operatorname{Var}\left[R_{i j}\right]=1$, and with the subgaussian norm $\kappa=\left\|R_{i,j}\right\|_{\psi_2}$. Here and below in this theorem, the notation $O(\cdot)$ depends on the same absolute constant. Without loss of generality, we assume that $m < d$ ( or equivalently, we can assume that $\sqrt{\frac{O(p\log(K))}{d}}<1$ and $\epsilon \in (\sqrt{\frac{O(p\log(K))}{d}},1)$). Let $\Pi\in\mathbb{R}^{m\times d}$ be a random projection matrix drawn from $D_{\epsilon,\delta}$. By Lemma \ref{thm: extreme singular subgaussian}, there exists some constant $C_{\kappa}^{2}>0$ that only depends on $\kappa$ such that the inequality (\ref{equ: subgaussian singular value estimation}) holds for any $t\geq 0$ with probability over $1-2 \exp \left(-t^2\right)$.
Define
\begin{equation}
\begin{array}{c}
C_0=\frac{\sqrt{O(p\log K)}}{\sqrt{d} + C_{\kappa}^{2}t},\;
\tilde{\epsilon}_{\sup}=rC_0\sqrt{\frac{r^2C_0^2}{4}+C_{\kappa}^{2}C_0+1}-C_{\kappa}^{2}C_0-\frac{r^2C_0^2}{2},\\
\tilde{\epsilon}_{\sup2}=r_{2}C_0\sqrt{\frac{r_{2}^2C_0^2}{4}+C_{\kappa}^{2}C_0+1}-C_{\kappa}^{2}C_0-\frac{r_{2}^2C_0^2}{2},\;\tilde{\epsilon}_{\min}=\sqrt{\frac{O(p\log K)}{d}}.
\end{array}
\end{equation}
Denote the optimal solution of the model (\ref{model: PSON}) with $\Pi$ and $\gamma \geq 0$ by $\left\{\hat{\mathbf{x}}_i^*(\gamma)\right\}_{i=1}^n$ and define the map $\hat{\phi}_{\gamma}\left(\mathbf{a}_i\right)=\hat{\mathbf{x}}_i^*$. Then, we have:

1. If $r>\frac{1+C_{\kappa}^{2}+\frac{ C_{\kappa}^{2}t}{\sqrt{d}}}{\sqrt{1-\tilde{\epsilon}_{\min}}}$, then  $\tilde{\epsilon}_{\min}<\tilde{\epsilon}_{\sup}$. For any $
\epsilon\in[\tilde{\epsilon}_{\min},\tilde{\epsilon}_{\sup})$,
and $\hat{\gamma} \in\left[\bar{S}(m,d,t)\gamma_{\min},\sqrt{1-\epsilon}\gamma_{\max }\right)$, with probability over $1-\frac{1}{K^{p-2}}-2\exp(-t^2)$, the map $\hat{\phi}_{\gamma}$ perfectly recovers $\mathcal{V}$.

2. If $r_{2}>\frac{1+C_{\kappa}^{2}+\frac{ C_{\kappa}^{2}t}{\sqrt{d}}}{\sqrt{1-\tilde{\epsilon}_{\min}}}$, then $\tilde{\epsilon}_{\min}<\tilde{\epsilon}_{\sup2}$. For any $\epsilon\in[\tilde{\epsilon}_{\min},\tilde{\epsilon}_{\sup2})$, and $\hat{\gamma} \in\left[\bar{S}(m,d,t)\gamma_{\min},\sqrt{1-\epsilon}\gamma_{\max2 }\right)$, with probability over $1-\frac{1}{K^{p-2}}-2\exp(-t^2)$,
the map $\hat{\phi}_{\gamma}$ recovers a non-trivial coarsening of $\mathcal{V}$.
\end{theorem}

\begin{proof}
With probability over $1-\frac{1}{K^{p-2}}-2\exp(-t^2)$, we have

\begin{itemize}
    \item[(i)] The centroids $\left\{ \Pi\mathbf{a}^{(0)}, \dots, \Pi\mathbf{a}^{(K)}\right\}$ of the embedded data are distinct.
    \item[(ii)]
     The parameters $\hat{\gamma}_{\min}$, $\hat{\gamma}_{\max}$, and $\hat{\gamma}_{\max2}$ defined in (\ref{equ: gamma_hat min &max}) satisfy the following inequalities:
    \[
    \begin{array}{c}
    \hat{\gamma}_{\min}\leq s_{1}(\Pi)\gamma_{\min}\leq\bar{S}(m,d,t)\gamma_{\min},\quad
    \sqrt{1-\epsilon}\gamma_{\max } \leq \hat{\gamma}_{\max },\quad \sqrt{1-\epsilon}\gamma_{\max2 } \leq \hat{\gamma}_{\max2}.
    \end{array}
    \]
\end{itemize}
The above implies that
\begin{equation}
\label{equ: hat gamma min&max and gamma min&max log(K)}
\begin{array}{c}
\left[\bar{S}(m,d,t)\gamma_{\min}, \sqrt{1-\epsilon}\gamma_{\max}\right)\subseteq\left[\hat{\gamma}_{\min},\hat{\gamma}_{\max}\right),\\
\left[\bar{S}(m,d,t)\gamma_{\min}, \sqrt{1-\epsilon}\gamma_{\max2}\right)\subseteq\left[\hat{\gamma}_{\min},\hat{\gamma}_{\max2}\right).
\end{array}
\end{equation}

Now, we prove the first part of the theorem. We claim here that it is sufficient to show: if $r>\frac{1+C_{\kappa}^{2}+\frac{ C_{\kappa}^{2}t}{\sqrt{d}}}{\sqrt{1-\tilde{\epsilon}_{\min}}}$, then  $\tilde{\epsilon}_{\min}<\tilde{\epsilon}_{\sup}$, and for any $
\epsilon\in[\tilde{\epsilon}_{\min},\tilde{\epsilon}_{\sup})$, the interval $\left[\bar{S}(m,d,t)\gamma_{\min},\sqrt{1-\epsilon}\gamma_{\max }\right)$ is nonempty. In fact, if $\left[\bar{S}(m,d,t)\gamma_{\min},\sqrt{1-\epsilon}\gamma_{\max }\right)$ is nonempty, then by the first inclusion of (\ref{equ: hat gamma min&max and gamma min&max log(K)}), $\left[\hat{\gamma}_{\min},\hat{\gamma}_{\max}\right)$ is nonempty. Applying Theorem \ref{thm: WSON} to the embbedded data $\Pi A$ implies that for any
$\hat{\gamma} \in \left[\hat{\gamma}_{\min},\hat{\gamma}_{\max}\right)$, the map $\hat{\phi}_{\hat{\gamma}}$ perfectly recovers $\mathcal{V}$.

On the one hand, by definition of $\tilde{\epsilon}_{\min}$ and $C_{0}$, we have
$\frac{1}{\sqrt{d}}=\frac{\tilde{\epsilon}_{\min}}{\sqrt{O(p\log(K))}}$, and
$C_{0}^{-1}=\tilde{\epsilon}_{\min}^{-1}+\frac{ C_{\kappa}^{2}t}{\sqrt{O(P\log(k))}}$.
As a result,
\[
\begin{array}{lcl}
r>\frac{C_{\kappa}^{2}+1+\frac{ C_{\kappa}^{2}t}{\sqrt{d}}}{\sqrt{1-\tilde{\epsilon}_{\min}}} &\implies& r>\frac{C_{\kappa}^{2}+\tilde{\epsilon}_{\min}\tilde{\epsilon}_{\min}^{-1}+\frac{\tilde{\epsilon}_{\min} C_{\kappa}^{2}t}{\sqrt{O(P\log(k))}}}{\sqrt{1-\tilde{\epsilon}_{\min}}}
\\ & \implies &
r>\frac{C_{\kappa}^{2}+\tilde{\epsilon}_{\min}\left(\tilde{\epsilon}_{\min}^{-1}+\frac{ C_{\kappa}^{2}t}{\sqrt{O(P\log(k))}}\right)}{\sqrt{1-\tilde{\epsilon}_{\min}}}
\\
& \implies & r>\frac{C_{\kappa}^{2}+\tilde{\epsilon}_{\min}C_{0}^{-1}}{\sqrt{1-\tilde{\epsilon}_{\min}}}
\\
& \implies &
C_{\kappa}^{2}+\tilde{\epsilon}_{\min}C_{0}^{-1}<\sqrt{1-\tilde{\epsilon}_{\min}}r
\\
& \implies & \left( C_{0}^{-1}\tilde{\epsilon}_{\min}+C_{\kappa}^{2}\right)^{2}<(1-\tilde{\epsilon}_{\min})r^{2}
\\
& \implies & \tilde{\epsilon}_{\min}^{2}+C_{0}
\left(2C_{\kappa}^{2}+r^{2}C_{0}\right)\tilde{\epsilon}_{\min}+C_{0}^{2}\left(-r^{2}+C_{\kappa}^{4}\right)<0.
\end{array}
\]
In other words, $\tilde{\epsilon}_{\min}$ satisfies the following inequality
\[
\begin{array}{c}
x^{2}+C_{0}
\left(2C_{\kappa}^{2}+r^{2}C_{0}\right)x+C_{0}^{2}\left(-r^{2}+C_{\kappa}^{4}\right)<0.
\end{array}
\]
It is not difficult to check the solutions to the above inequality is $x\in(x_{1},x_{2})$,
where
\[
\begin{array}{c}
x_{1}=-rC_0\sqrt{\frac{r^2C_0^2}{4}+C_{\kappa}^{2}C_0+1}-C_{\kappa}^{2}C_0-\frac{r^2C_0^2}{2}<0,\\
x_{2}=rC_0\sqrt{\frac{r^2C_0^2}{4}+C_{\kappa}^{2}C_0+1}-C_{\kappa}^{2}C_0-\frac{r^2C_0^2}{2}\in(0,1).
\end{array}
\]
One may realize that $x_{2}=\tilde{\epsilon}_{\sup}$.
This implies that $\tilde{\epsilon}_{\min}<\tilde{\epsilon}_{\sup}$.

 On the other hand, we have
\[
\begin{array}{lcl}
\epsilon\in[\tilde{\epsilon}_{\min},\tilde{\epsilon}_{\sup})\subseteq (x_{1},\tilde{\epsilon}_{\sup})& \implies & \epsilon^{2}+C_{0}
\left(2C_{\kappa}^{2}+r^{2}C_{0}\right)\epsilon+C_{0}^{2}\left(-r^{2}+C_{\kappa}^{4}\right)<0
\\
& \implies & \left( C_{0}^{-1}\epsilon+C_{\kappa}^{2}\right)^{2}<(1-\epsilon)r^{2}
\\
& \implies & \left( C_{0}^{-1}\epsilon+C_{\kappa}^{2}\right)<\sqrt{1-\epsilon}r
\\
& \implies & \left( C_{0}^{-1}\epsilon+C_{\kappa}^{2}\right)\gamma_{\min}<\sqrt{1-\epsilon}\gamma_{\max}
\\
& \implies & \left(\frac{\sqrt{d}+C_{\kappa}^{2}t}{\sqrt{O(p\log(K)}}\epsilon+C_{\kappa}^{2}\right)\gamma_{\min}<\sqrt{1-\epsilon}\gamma_{\max}
\\
& \implies & \left(\frac{\sqrt{d}+C_{\kappa}^{2}t}{\sqrt{m}}+C_{\kappa}^{2}\right)\gamma_{\min}<\sqrt{1-\epsilon}\gamma_{\max}
\\
& \implies & \bar{S}(m,d,t)\gamma_{\min}<\sqrt{1-\epsilon}\gamma_{\max }.
\end{array}
\]

Thus, we have proved the first part of the theorem. The second part of the theorem can be proved in a similar way.
\end{proof}

\begin{remark}
Here, we want to make some remarks on the obtained results.

\begin{enumerate}
    \item The embedding dimension in Theorem \ref{thm_PSON3} is independent of the number of data points $n$, which is important for clustering an extremely large number of data points.

    \item The results of this theorem and Theorem \ref{thm_PSON2} further demonstrate that the ratio $r = \gamma_{\max}/\gamma_{\min}$ is a data scale-invariant measure to characterize the difficulty of clustering a given collection of data. Since the embedding dimension of the JL lemma depends on O($\epsilon^{-2}$) and $\epsilon \in (0, 1)$, the value $\tilde{\epsilon}_{\min}$ (and $\epsilon_{\min}$) can be interpreted as the lowest possible dimension reduction ratio obtained by the JL lemma. Since the JL lemma is optimal if the $\epsilon$-isometry mapping is linear, thus, the condition $r>\frac{1+C_{\kappa}^{2}+\frac{ C_{\kappa}^{2}t}{\sqrt{d}}}{\sqrt{1-\tilde{\epsilon}_{\min}}}$ in Theorem \ref{thm_PSON3} (and $r > \sqrt{\frac{1 + \epsilon_{\min}}{1 - \epsilon_{\min}}}$ in Theorem \ref{thm_PSON2}) shows that the dimension reduction results obtained in this paper are intrinsically depending on the difficulties of clustering the input data.

    \item For the K-means model, \citet{JL-Kmeans15} proved that the cost can be preserved up to a $(9 + \epsilon)$ approximation bounds if the embedding dimension $m = O(\epsilon^{-2}\log K)$. This bound has been improved to $(1 + \epsilon)$ if the embedding dimension is $m = O(\epsilon^{-2}\log(K/\epsilon))$ \citep{k-means-improved}. However, it is still unknown whether the randomly projected K-means model can preserve the cluster membership assignments or not.
\end{enumerate}
\end{remark}

\section{Numerical Experiments}
\label{sec: numerical results}
In this section, we present extensive numerical experiment results to show the practical performance of our model (\ref{model: PSON}). We first consider high-dimensional data randomly generated from a mixture of spherical Gaussians $\mathcal{N}(\boldsymbol{\mu}_{k},\sigma_{k}^{2}I_{d})$  with $K$ distinct means $\boldsymbol{\mu}_1, \ldots, \boldsymbol{\mu}_K \in \mathbb{R}^d $.

In the realization of dimension reduction, by default, we randomly sample $\Pi\in\mathbb{R}^{m\times d}$ through
\begin{equation}
\label{equ: Pi Gaussians}
\Pi=\frac{1}{\sqrt{m}}G \in \mathbb{R}^{m \times d},
\end{equation}
where $G_{ij}$ are sampled from i.i.d. standard normal distribution, and $m=ceil(C\epsilon^{-2}\log(n))$. Here, $C>0$ is a constant, and $\epsilon\in(0,1)$ is the distortion parameter, which will be specified in the experiments.

In this section, we will set the weights of the convex clustering model as follows:
\begin{equation}
\label{equ: weight-setting}
w_{i j}= \begin{cases}\exp (-\phi\left\|\mathbf{a}_i-\mathbf{a}_j\right\|^2) & \text { if }(i, j) \in \mathcal{E}, \\ 0 & \text { otherwise },\end{cases}
\end{equation}
where $\mathcal{E} := \{(i, j) ~\mid~ \mbox{if $\mathbf{a}_i$ (or $\mathbf{a}_j$) is in $\mathbf{a}_j$'s (or $\mathbf{a}_i$'s) k-nearest neighbors}, 1 \leq i \neq j \leq n \}$. We set $\phi=\frac{1}{d}$ by default to rescale the weights and $k$ will be specified in the experiments.

We adopt the semismooth Newton based augmented Lagrangian method (SSNAL) \citep{sun2021convex}, which is a state-of-the-art algorithm for solving models (\ref{model: WSON}) and (\ref{model: PSON}). We adopt the duality gap as the stopping criterion (see \citep{yuan2021dimension} for details) with a tolerance $\epsilon_{\rm tol} = 10^{-6}$.

We organize our numerical experiment results as follows: In Section \ref{sec: numerical-sec-1}, we first justify the quality of the random projection matrix $\Pi$ for preserving the pairwise distances for the data points and centroids. After that, we verify the recovery guarantees of the model (\ref{model: PSON}). We further compare the cluster recovery performance of the model (\ref{model: PSON}) to the randomly projected K-means model (RP K-means). In Section \ref{sec: numerical-verification-logk}, we will numerically demonstrate that the embedding dimension can be $O(\epsilon^{-2}\log(K))$. In Section \ref{sec: numerical-sec2}, we test the robustness of the model (\ref{model: PSON}) with different problem scales and embedding dimensions.  Lastly, we test the performance of the model on real data in Section \ref{sec: numerical-true}.

\subsection{Numerical Verification for the Randomly Projected Convex Clustering Model with $m = O(\epsilon^{-2}\log(n))$}
\label{sec: numerical-sec-1}
In this section, we verify the theoretical performance of the model (\ref{model: PSON})  by conducting numerical experiments on one simulated balanced Gaussian data $A\in\mathbb{R}^{2000\times 1000}$. Data $A$ is generated from a mixture of $K=20$ spherical Gaussians $\mathcal{N}(\mathbf{e}_k,0.005I_{2000})$ with equal probability $w_k=\frac{1}{20}$, for all $k=1,\ldots,20$. Here, $\mathbf{e}_k \in \mathbb{R}^{2000}$ is the $k$-th column of the identity matrix $I_{2000}$. Note that we know the true cluster assignments of the simulated data. Let $X_A=\{\mathbf{a}_{i}-\mathbf{a}_{j} ~|~ 1\leq i<j\leq n\}$, $X_{\alpha}=\{\mathbf{a}_{i}-\mathbf{a}_{j} ~|~ i,j\in I_{\alpha},i\neq j\},\alpha=1,..., K$, and $X_{\mathcal{C}(A)}=\{\mathbf{a}^{(\alpha)}-\mathbf{a}^{(\beta)} ~|~ 1\leq \alpha<\beta\leq K\}$. Let $X_{\mathcal{V}}=\cup_{\alpha=1}^{20}{X_{\alpha}}$. The size of $X_A$ is $\mathcal{C}(1000,2)=499500$, the size of $X_{\mathcal{V}}$ is $\sum_{\alpha=1}^{20}\mathcal{C}(n_{\alpha},2)=24926$, and the size of $X_{\mathcal{C}(A)}$ is $\mathcal{C}(20,2)=210$. The visualization of this data set is in Figure \ref{fig: MSGd2000n1000K20sigma}. For all the visualizations of the high-dimensional data points in this paper, we adopt the t-SNE \citep{van2008visualizing} to project them to $\mathbb{R}^3$. Motivated by the assumptions of the recovery guarantees, we will set the weights $w_{i j}$ as (\ref{equ: weight-setting}) with a graph
\begin{equation}
\label{equ: weight graph for A}
\begin{aligned}
\mathcal{E}_{A} := &\cup_{i=1}^{1000}\{(i, j) ~\mid~ \mbox{if $\mathbf{a}_i$ (or $\mathbf{a}_j$) is in $\mathbf{a}_j$'s (or $\mathbf{a}_i$'s) 20-nearest neighbors}, 1 \leq i \neq j \leq 1000 \}\\&\cup_{\alpha=1}^{20}\{(i, j) ~\mid~ i, j \in I_\alpha, i\neq j\}.
\end{aligned}
\end{equation}

\subsubsection{Quality of the Random Projection Matrix}

We will verify the robustness of $\Pi$ for pair-wise distance preservation. For this purpose, we will generate the projection matrices $\Pi$ following (\ref{equ: Pi Gaussians}) with $m=\operatorname{ceil}(9\epsilon^{-2}\log(1000))$ and $\epsilon\in\{0.2,0.4,0.6,0.8,0.95\}$. In other words, we will test the random projection matrices with $m\in\{1555,389,173,98,69\}$. We first randomly generate a projection matrix and visualize the embedded data for each $m$ in Figure \ref{fig: MSGd2000n1000K20eps02C9m1555}, \ref{fig: MSGd2000n1000K20eps04m389}, \ref{fig: MSGd2000n1000K20eps06m173}, \ref{fig: MSGd2000n1000K20eps08_m98}, and \ref{fig: MSGd2000n1000K20eps095m69}, respectively.
\begin{figure}[tbhp]
\centering
\subfloat[$d=2000$]{\label{fig: MSGd2000n1000K20sigma}\includegraphics[width=0.33\textwidth]{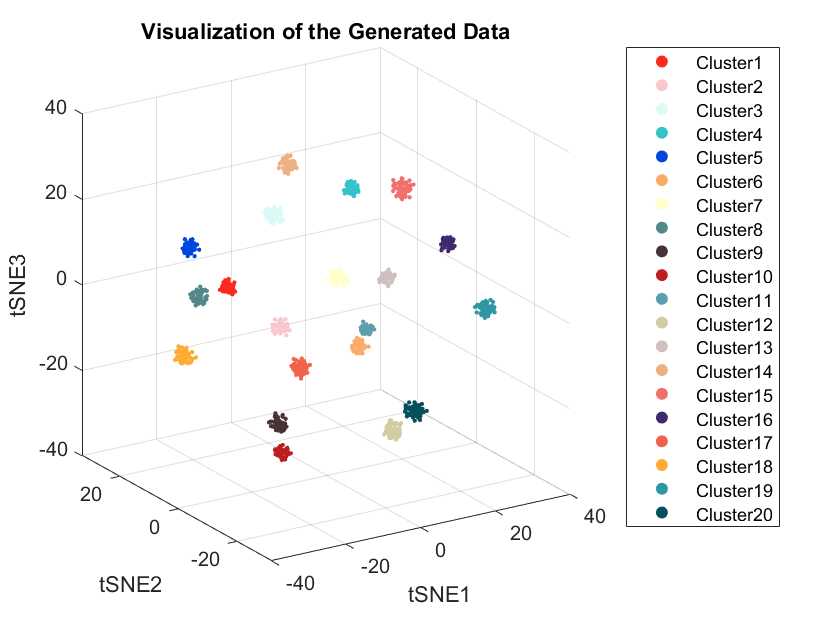}}
\subfloat[ $m=1555$ ($\epsilon= 0.2$)]{\label{fig: MSGd2000n1000K20eps02C9m1555}\includegraphics[width=0.33\textwidth]{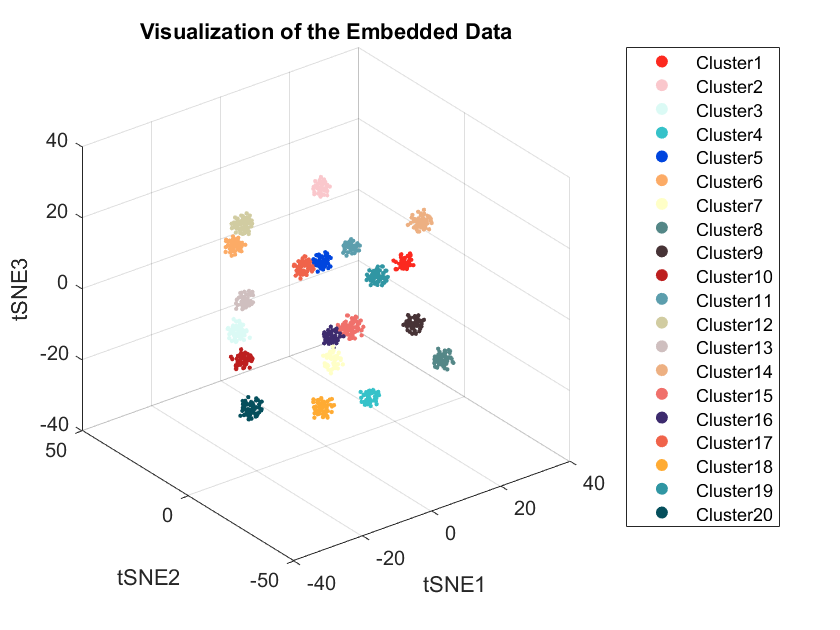}}
\subfloat[$m=389$ ($\epsilon= 0.4)$]{\label{fig: MSGd2000n1000K20eps04m389}\includegraphics[width=0.33\textwidth]{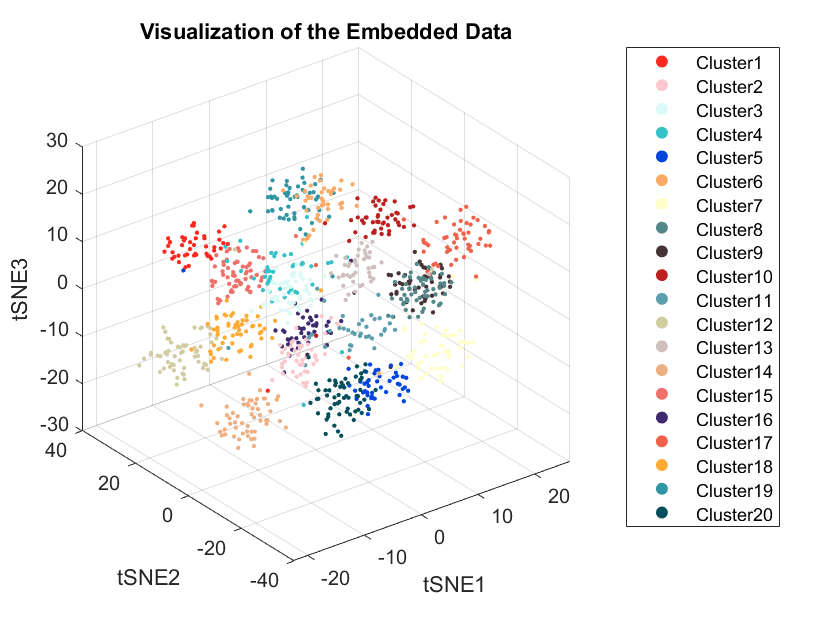}}

\subfloat[$m=173$ ($\epsilon= 0.6)$]{\label{fig: MSGd2000n1000K20eps06m173}\includegraphics[width=0.33\textwidth]{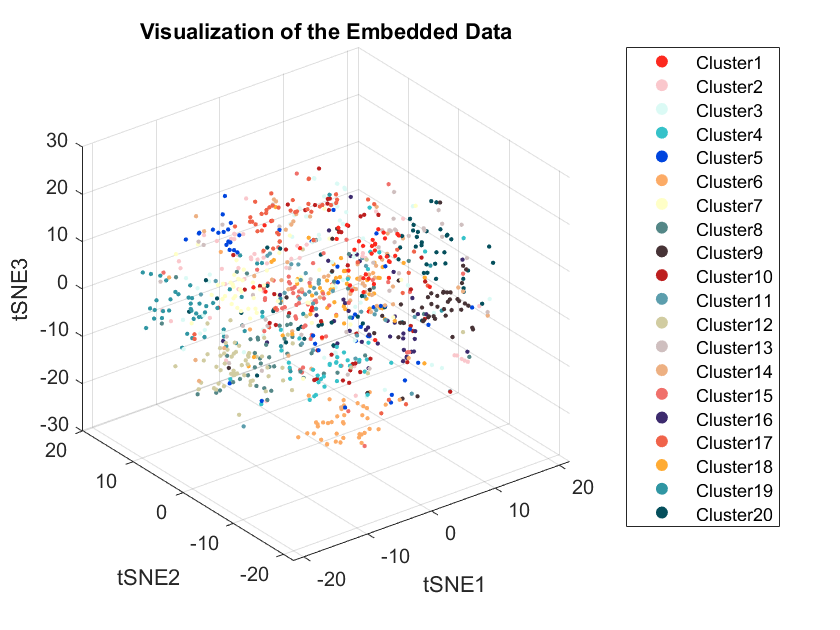}}
\subfloat[$m=98$ ($\epsilon= 0.8)$]{\label{fig: MSGd2000n1000K20eps08_m98}\includegraphics[width=0.33\textwidth]{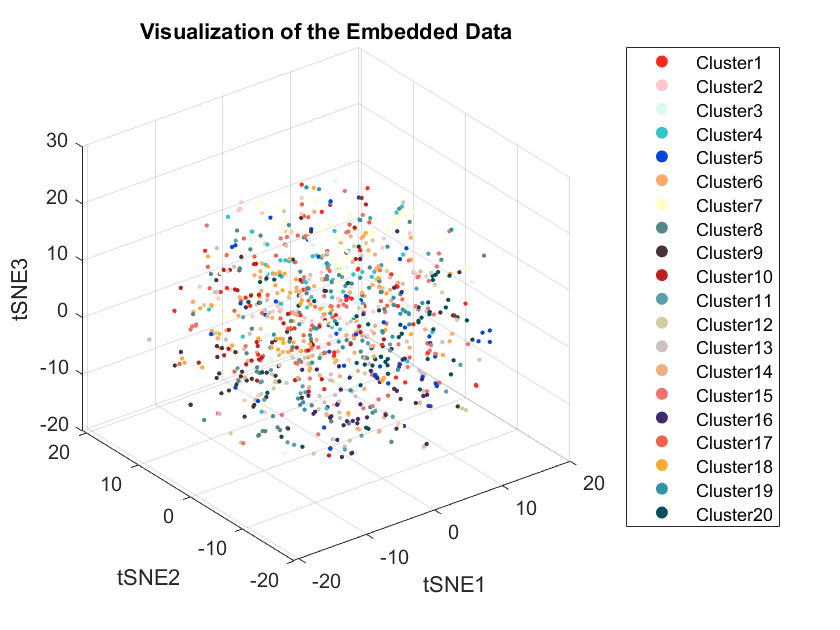}}
\subfloat[$m=69$ ($\epsilon= 0.95)$]{\label{fig: MSGd2000n1000K20eps095m69}\includegraphics[width=0.33\textwidth]{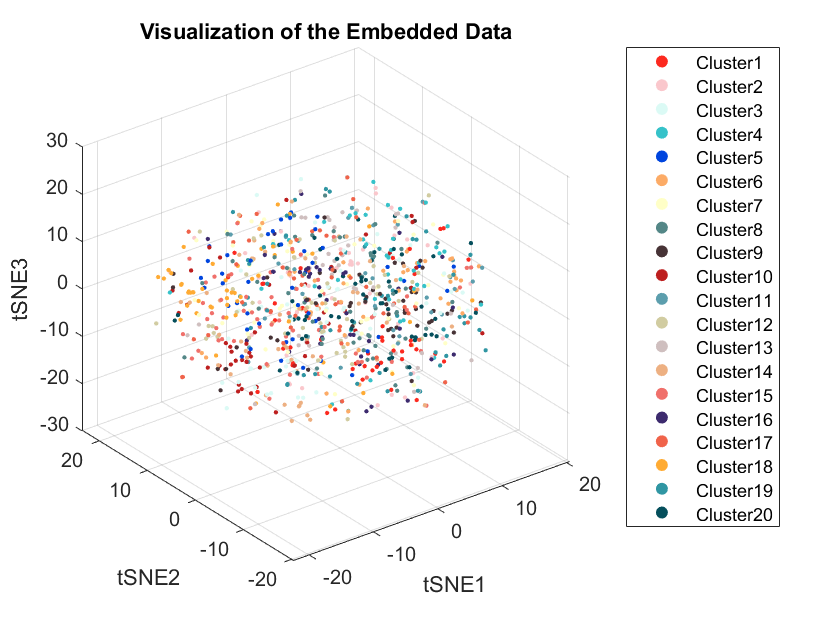}}
\caption{Visualization for $A$ and five embedded data.}
\label{Fig:view_MSG and its embeddding}
\end{figure}
From the figures, one may see that as the distortion parameter $\epsilon$ increases (in other words, $m$ decreases), different clusters in the embedded data become less separate, which is intuitive. Moreover, we can observe that the random projection matrix can preserve the pairwise distances structure of the input data $A$ very well if we set a relatively small distortion parameter $\epsilon$. To further demonstrate the robustness, we will randomly generate 1000 independent samples of the random projection matrix $\Pi$ for every parameter setting, and test the successful probability of the squared-norm preservation of the points in the sets $X_A$, $X_{\mathcal{V}}$, and $X_{\mathcal{C}(A)}$ within the desired distortion $\epsilon$. The results are summarized in Table \ref{tab: check for pwd}. The results demonstrate the robustness of the random projection matrices for pair-wise distance preservation. On the one hand, the square-norm can be preserved for almost all points (with a percentage over $99.999\%$). On the other hand, the success rate for a random projection matrix to preserve the square-norm for all the points in $X_{\mathcal{V}}$ and the centroids $X_{\mathcal{C}}$ are very high.

\begin{table}[tbhp]
\caption{The numerical performance of the random projection matrix $\Pi$ for preserving the squared norm of the points in $X_A$, $X_{\mathcal{V}}$, and $X_{\mathcal{C}(A)}$ within the desired distortion. In the table, $p_{X_A}$, $p_{X_{\mathcal{V}}}$, and $p_{X_{\mathcal{C}(A)}}$ are the successful probability for preserving the squared norm of all the points. $X_A\%$, $X_{\mathcal{V}}\%$, and $X_{\mathcal{C}(A)}\%$ are the average percentage of the points whose squared norm are preserved within the desired distortion.}
\label{tab: check for pwd}
{\footnotesize
\begin{center}
\begin{tabular}{|c|c|c|c|c|c|c|} \hline
 \bf{Dimension (distortion)} &$p_{X_A}$& $X_A\%$&$p_{X_{\mathcal{V}}}$& $X_{\mathcal{V}}\%$& $p_{X_{\mathcal{C}(A)}}$& $X_{\mathcal{C}(A)}\%$ \\ \hline
$m=1555\quad (\epsilon= 0.2$) &  950/1000 & 99.999\% & 1000/1000 &99.999\% & 1000/1000 &100\%\\\hline
$m=389\quad (\epsilon= 0.4$) &  855/1000 & 99.999\%  & 993/1000 & 99.999\% & 1000/1000 &100\%\\\hline
$m=173\quad (\epsilon= 0.6$) &  705/1000  & 99.999\% & 982/1000 & 99.999\%& 1000/1000&100\%\\\hline
$m=98\quad (\epsilon= 0.8$) &  501/1000& 99.999\% & 951/1000& 99.999\%& 1000/1000 &100\%\\\hline
$m=69\quad (\epsilon= 0.95$) &  248/1000 & 99.999\%& 907/1000& 99.999\% & 1000/1000 &100\%\\\hline
\end{tabular}
\end{center}
}
\end{table}
\subsubsection{Verification of the Recovery Guarantees of the Randomly Projected Convex Clustering Model}
Next, we will verify the recovery guarantees of the model (\ref{model: PSON}) established in Theorem \ref{thm_PSON2}. Since the effectiveness of the random projection matrix $\Pi$ for pair-wise distance preservation has already been verified, now, we will randomly sample a projection matrix $\Pi$ for each $m$ in the experiments described below. We first compute the upper bound $\gamma_{\max}$ and the lower bound $\gamma_{\min}$ of $\gamma$ defined by (\ref{equ: gamma min &max}) and their ratio $r=\frac{\gamma_{\max}}{\gamma_{\min}}$ on the original data $A$. The values are $$\gamma_{\min}=0.1620,\quad\gamma_{\max}=1.2474,\quad r=7.6985,$$
which imply that the model (\ref{model: WSON}) with our designed weights $w_{ij}$ can perfectly recover the true cluster membership of $A$ for any $\gamma\in[0.1620,1.2474)$. The large ratio $r$ implies the feasibility of the model (\ref{model: PSON}) with some suitable $\epsilon\in(0,1)$ under the same weights $w_{ij}$. We then estimate the values $\epsilon_{\min}$ and $\epsilon_{\sup}$ defined in Theorem \ref{thm_PSON2}, which are
$$
\epsilon_{\min}=0.1763,\quad\epsilon_{\sup}=0.9668.
$$
The results in Theorem \ref{thm_PSON2} imply that for $0.1763\leq\epsilon<0.9668$, and ${\gamma} \in\left[\sqrt{1+\epsilon}\gamma_{\min}, \sqrt{1-\epsilon}\gamma_{\max}\right)$, the model (\ref{model: PSON}) with $m=O(9\epsilon^{-2}\log(1000))$ can perfectly recover the true cluster membership of $A$ with high probability. Here, we take $m=ceil(9\epsilon^{-2}\log(1000))$.

Since the estimated valid interval of distortions is $(\epsilon_{\min}, \epsilon_{\sup}) = (0.1763,0.9668)$, we choose $\epsilon \in \{0.2,0.4,0.6,0.8,0.95\}$ for verification. The corresponding embedding dimensions are $m \in \{1555,389,173,98,69\}$. To verify the recovery guarantees of the models (\ref{model: WSON}) and (\ref{model: PSON}), we will generate a clustering path of the model (\ref{model: WSON}) on the original data $A$ and a clustering path of the model (\ref{model: PSON}) on the embedded data for each $m$. In particular, we will generate all clustering paths with $\gamma \in [10:-0.1:0.1]$. We will compute the number of clusters $K$, the rand index, and the adjusted rand index against $\gamma$ on the clustering paths. The results are shown in Figure \ref{Fig: clustering and (adjusted) rand index path}.
\begin{figure}[tbhp]
\centering
\subfloat[Number of clusters $K$]{\label{fig: 6clustering path MSGd2000n1000K20sigma}\includegraphics[width=1\textwidth]{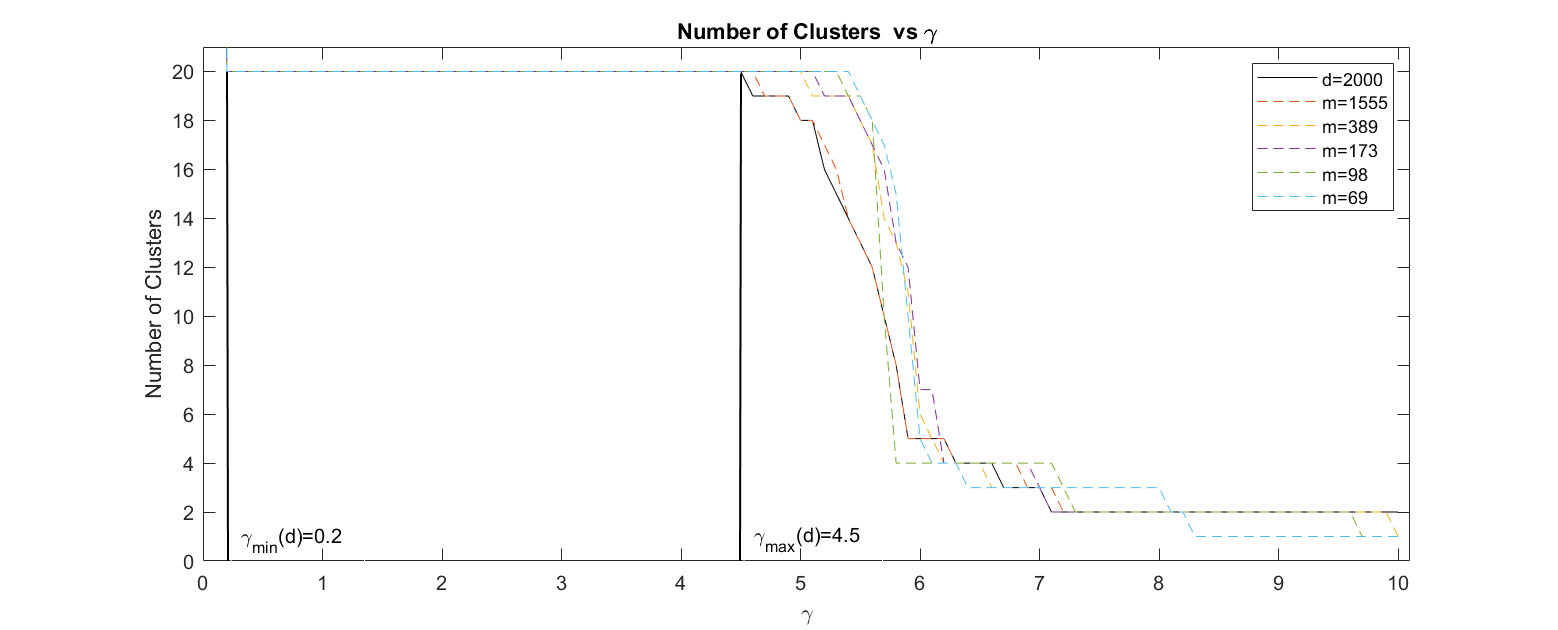}}

\subfloat[Practical upper bound of $\gamma$ for perfect recovery]{\label{fig: 6clustering path gamma max}\includegraphics[width=1\textwidth]{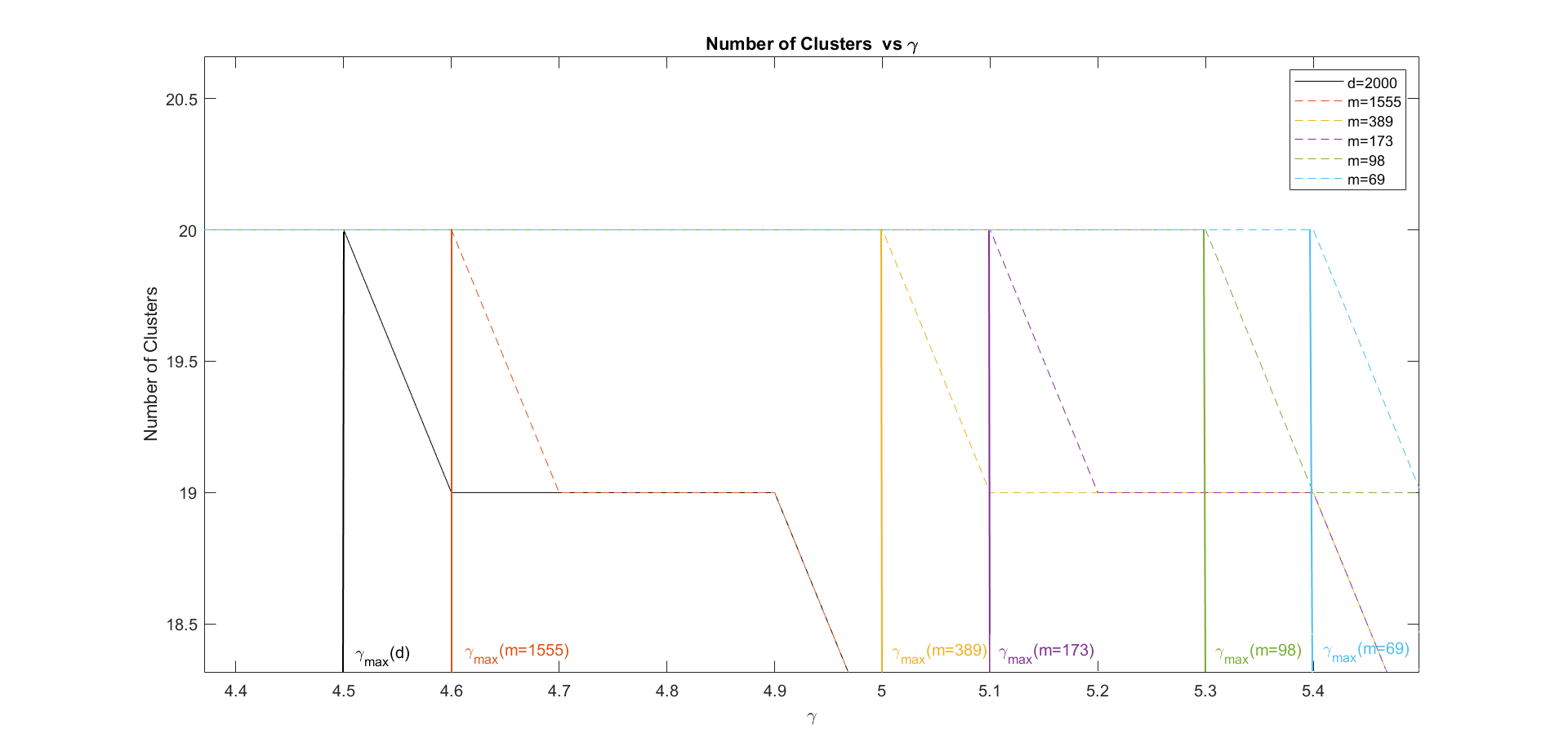}}

\subfloat[rand index on the path]{\label{fig: 6ri path MSGd2000n1000K20sigma}\includegraphics[width=0.5\textwidth]{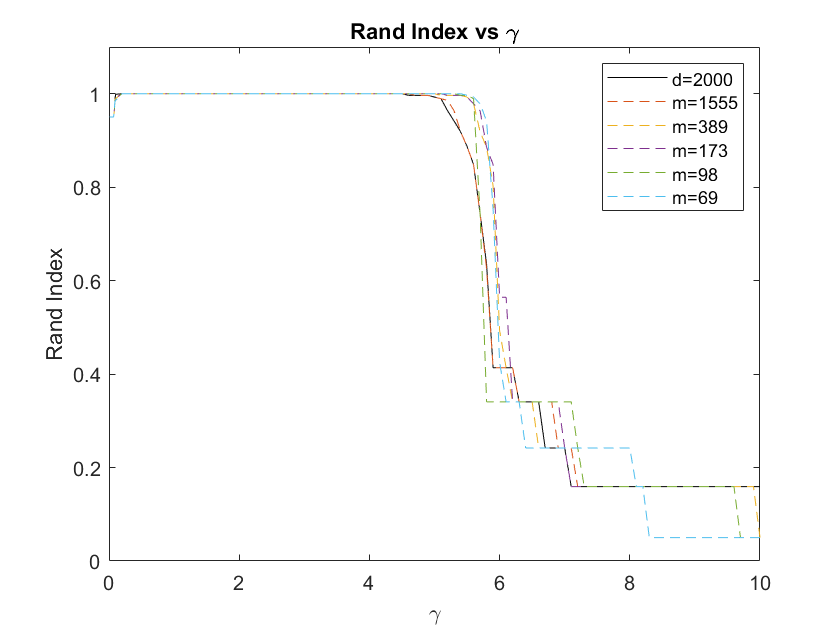}}
\subfloat[adjusted rand index on the path]{\label{fig: 6ari path MSGd2000n1000K20sigma}\includegraphics[width=0.5\textwidth]{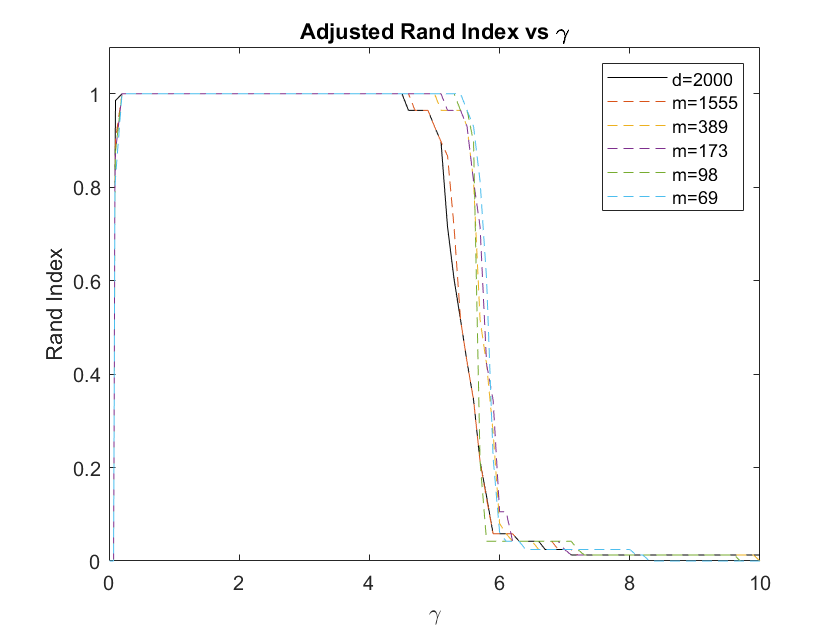}}
\caption{Clustering performance on the clustering path.}
\label{Fig: clustering and (adjusted) rand index path}
\end{figure}
To better verify the recovery guarantees, we compute the estimated range  $\left[\sqrt{1+\epsilon}\gamma_{\min},\sqrt{1-\epsilon}\gamma_{\max}\right)$ in Theorem {\ref{thm_PSON2}} for perfect recovery for different $m$ in Table \ref{tab: theoretical estimation of gamma}.
\begin{table}[tbhp]
\caption{Estimated ranges of $\gamma$ for perfect recovery guarantees of RPCCM. The range $\left[\sqrt{1+\epsilon}\gamma_{\min},\sqrt{1-\epsilon}\gamma_{\max}\right)$ is estimated using Theorem \ref{thm_PSON2} by the model (\ref{model: PSON}) and the range $\left[\hat{\gamma}_{\min},\hat{\gamma}_{\max}\right)$ defined by (\ref{equ: gamma_hat min &max}) is implicitly estimated using Theorem \ref{thm: WSON} by (\ref{model: WSON}).}
\label{tab: theoretical estimation of gamma}
{\footnotesize
\begin{center}
\begin{tabular}{|c|c|c|} \hline
 \bf{Dimension (distortion)} &  $\left[\sqrt{1+\epsilon}\gamma_{\min},\sqrt{1-\epsilon}\gamma_{\max}\right)$ &  $\left[\hat{\gamma}_{\min},\hat{\gamma}_{\max}\right)$
 \\\hline
$m=1555\quad (\epsilon= 0.2$) &   [0.1775,1.1157]& [0.1631,1.2334)
\\\hline
$m=389\quad (\epsilon= 0.4$) &   [0.1917,0.9669)&[0.1699,1.2680)
\\\hline
$m=173\quad (\epsilon= 0.6$) & [0.2049, 0.7889)&[0.1610,1.1783)
\\\hline
$m=98\quad (\epsilon= 0.8$) & [0.2174, 0.5578)&[0.1618,1.2101)
\\\hline
$m=69\quad (\epsilon= 0.95$) & [0.2263, 0.2789)&[0.1707, 1.2443)
\\\hline
\end{tabular}
\end{center}
}
\end{table}

From the results in Figure \ref{Fig: clustering and (adjusted) rand index path} and Table \ref{tab: theoretical estimation of gamma}, we can see that the model (\ref{model: PSON}) indeed performs perfect cluster recovery when $\gamma$ is chosen in the interval $\left[\sqrt{1+\epsilon}\gamma_{\min}, \sqrt{1-\epsilon}\gamma_{\max}\right)$.

In a word, the recovery guarantees of the convex clustering model (\ref{model: WSON}) on the original data $A$ can be preserved by the model (\ref{model: PSON}) with a much smaller dimension and the performance of the model (\ref{model: PSON}) is attractive in practice.

\begin{remark}
We want to add a remark here on the empirical performance of the model (\ref{model: PSON}). As shown in Table \ref{tab: theoretical estimation of gamma}, the interval $[\hat{\gamma}_{\min}, \hat{\gamma}_{\max})$ of the model (\ref{model: PSON}) for the perfect recovery can be larger. The empirical performance can be robust with respect to the embedding dimension. This can be demonstrated by the results in Figure \ref{Fig: clustering and (adjusted) rand index path}.
\end{remark}

\subsubsection{Comparison between the Randomly Projected Convex Clustering Model and the Randomly Projected K-means Model}
To further demonstrate the superior performance of the model (\ref{model: PSON}), we compare the clustering performance between the model (\ref{model: PSON}) and the RP K-means on the data $A$. Since we know the true number of clusters is $K = 20$, we compare the clustering quality of the two models for $K \in \{16,17,18,19,20\}$. More specifically, we will compare their performance in terms of the rand index and the adjusted rand index against different numbers of clusters. For the implementation of K-means and RP K-means in this paper, we use the "kmeans" package from Matlab with parameters ’MaxIter’=10000 and ’Replicates’=30. We summarize the results in Table \ref{tab: RI RPCCM vs RPKmeans} and Table \ref{tab: ARI RPCCM vs RPKmeans}.

From the results in Table \ref{tab: RI RPCCM vs RPKmeans} and Table \ref{tab: ARI RPCCM vs RPKmeans}, we can see that the performance of the model (\ref{model: PSON}) is better and more robust than RP K-means, even when the number of clusters is not correctly classified ($K\in\{16,17,18,19\}$). Neither K-means nor RP K-means can perform a perfect recovery based on our experiments, and the recovery performance of RP K-means becomes less reliable as $m$ decreases.
As a comparison, the recovery results of the model  (\ref{model: WSON}) are robustly inherited by the model (\ref{model: PSON}), and the model (\ref{model: PSON}) with all five $m$ could perform perfectly recovery on $A$ with some suitable $\gamma$ on the path. The embedding dimension is as low as $m=69$, which can greatly reduce the computational cost.

\begin{table}[tbhp]
\caption{The rand index value against the number of clusters ($K=\{16,17,18,19,20\}$) for CCM and K-means, and RPCCM and RP K-means with each $m$ on data $A$. For CCM and RPCCM, if $K$ is identified by some $\gamma$ (maybe not unique) on the clustering path, we will pick an adjusted rand index value as the record. Otherwise, if there is no $\gamma$ on the clustering path such that $K$ is identified, we will denote it by '/'.}
\label{tab: RI RPCCM vs RPKmeans}
{\footnotesize
\begin{center}
\begin{tabular}{|c|c|c|c|c|c|c|} \hline
 \bf{Model} & $K=16$ & $K=17$ & $K=18$ &$K=19$& $K=20$ \\ \hline
 \bf{CCM} ($d=2000$) & 0.9637 & / & 0.9929 & 0.9965 & \bf{1.0000}\\\hline
 \bf{RPCCM} ($m=1555$) & 0.9637 & 0.9857 & 0.9929 & 0.9965 & \bf{1.0000}\\\hline
\bf{RPCCM} ($m=389$) &  /& 0.9786 & 0.9929 & 0.9965 & \bf{1.0000}\\\hline
\bf{RPCCM} ($m=173$) & 0.9637 & 0.9786 & 0.9929 & 0.9965 &\bf{1.0000}\\\hline

\bf{RPCCM} ($m=98$) & / & / & 0.9893 & 0.9965 & \bf{1.0000}\\\hline
\bf{RPCCM} ($m=69$) & / & 0.9786 & 0.9929 & 0.9965 & \bf{1.0000}\\\hline

\bf{K-means} ($d=2000$) & 0.9695 & 0.9702 & 0.9684 & 0.9733 & 0.9851 \\\hline
\bf{RP K-means} ($m=1555$) & 0.9619 & 0.9808 & 0.9804 &  0.9754 & 0.9836 \\\hline
\bf{RP K-means} ($m=389$) & 0.9620 & 0.9644 & 0.9824 & 0.9778& 0.9811\\\hline
\bf{RP K-means} ($m=173$)  & 0.9367 & 0.9401 & 0.9344  &  0.9458 & 0.9473 \\\hline
\bf{RP K-means} ($m=98$) & 0.9045 & 0.8992 & 0.9031 & 0.9122 & 0.9107 \\\hline
\bf{RP K-means} ($m=69$) & 0.8971 & 0.8995 & 0.9019 & 0.9029 & 0.9040\\\hline
\end{tabular}
\end{center}
}
\end{table}

\begin{table}[tbhp]
\caption{The adjusted rand index value against the number of clusters ($K=\{16,17,18,19,20\}$) for CCM and K-means, and RPCCM and RP K-means with each $m$ on data $A$. For CCM and RPCCM, if $K$ is identified by some $\gamma$ (maybe not unique) on the clustering path, we will pick an adjusted rand index value as the record. Otherwise, if there is no $\gamma$ on the clustering path such that $K$ is identified, we will denote it by '/'.}
\label{tab: ARI RPCCM vs RPKmeans}
{\footnotesize
\begin{center}
\begin{tabular}{|c|c|c|c|c|c|c|} \hline
 \bf{Model} & $K=16$ & $K=17$ & $K=18$ &$K=19$& $K=20$ \\ \hline
 \bf{CCM} ($d=2000$)  & 0.7154 & / & 0.9299 & 0.9645 & \bf{1.0000}\\\hline
 \bf{RPCCM} ($m=1555$)  & 0.7154 & 0.8670 & 0.9299 & 0.9645 & \bf{1.0000}\\\hline
\bf{RPCCM} ($m=389$)  & / & 0.8125 & 0.9299 & 0.9645 & \bf{1.0000}\\\hline
\bf{RPCCM} ($m=173$)  & 0.7154 & 0.8125 & 0.9299 & 0.9645 & \bf{1.0000}\\\hline
\bf{RPCCM} ($m=98$)  & / & / & 0.8975 & 0.9645 & \bf{1.0000}\\\hline
\bf{RPCCM} ($m=69$) & / & 0.8125 & 0.9299 & 0.9645 & \bf{1.0000}\\\hline

\bf{K-means} ($d=2000$) & 0.7493 & 0.7525 & 0.7355 & 0.7674 & 0.8578 \\\hline
\bf{RP K-means} ($m=1555$) & 0.6989 & 0.8284 &  0.8237 & 0.7801 & 0.8426 \\\hline
\bf{RP K-means} ($m=389$) & 0.6791 &  0.6949 & 0.8367 & 0.7901 & 0.8164\\\hline
\bf{RP K-means} ($m=173$) & 0.4669 & 0.4525 & 0.4053 & 0.4654 & 0.4964 \\\hline
\bf{RP K-means} ($m=98$) & 0.1284 & 0.0858 & 0.1007 & 0.1353 & 0.1293 \\\hline
\bf{RP K-means} ($m=69$) & 0.0585 & 0.0807 & 0.0494 & 0.0645 & 0.0759\\\hline
\end{tabular}
\end{center}
}
\end{table}

\subsection{Numerical Verification for the Randomly Projected Convex Clustering Model with $m = O(\epsilon^{-2}\log(K))$}
\label{sec: numerical-verification-logk}
In this section, we will further verify the recovery guarantees established in Theorem \ref{thm_PSON3} for the model (\ref{model: PSON}). In other words, we want to numerically verify that the embedding dimension $m$ of the model (\ref{model: PSON}) can be further improved from $O(\epsilon^{-2}\log(n))$ to $O(\epsilon^{-2}\log(K))$. For simplicity, we choose $m = ceil(10\epsilon^{-2}\log(n))$ and  $\tilde{m} = ceil(10\epsilon^{-2}\log(K))$, respectively. Here, $\epsilon\in(0,1)$ is some given distortion.

We will conduct experiments on a collection of data points $A^{\prime} := \{\mathbf{a}_{1}^{\prime}, \dots, \mathbf{a}_{10000}^{\prime}\}\subseteq\mathbb{R}^{100}$, where each $\mathbf{a}_{i}^{\prime}$ is randomly sampled from a balanced Gaussian mixture. In particular, we set $K=10$, $\boldsymbol{\mu}_{k}=\mathbf{e}_k$, $\sigma_{k}^{2}=0.1$, and $w_k=\frac{1}{10}$, for $k=1,\ldots,10$ for the Gaussian mixture. Let  $X_{\alpha}^{\prime}=\{\mathbf{a}_{i}^{\prime}-\mathbf{a}_{j}^{\prime} ~|~ i,j\in I_{\alpha},i\neq j\},\alpha=1,..., 10$, and $X_{\mathcal{C}(A)}^{\prime}={\{\mathbf{a}^{\prime}}^{(\alpha)}-{\{\mathbf{a}^{\prime}}^{(\beta)} ~|~ 1\leq \alpha<\beta\leq 10\}$, and denote $X_{\mathcal{V}}^{\prime}=\cup_{\alpha=1}^{10}{X_{\alpha}^{\prime}}$. Similarly, inspired by the assumptions in Theorem \ref{thm_PSON3}, we will set the weights $w_{i j}$ as (\ref{equ: weight-setting}) with a graph
\begin{equation}
\label{equ: weight graph A2}
\begin{aligned}
\mathcal{E}_{A^{\prime}} := &\cup_{i=1}^{10000}\{(i, j) ~\mid~ \mbox{if $\mathbf{a}_i^{\prime}$ (or $\mathbf{a}_j^{\prime}$) is in $\mathbf{a}_j$'s (or $\mathbf{a}_i^{\prime}$'s) 10-nearest neighbors}, 1 \leq i \neq j \leq 10000 \}\\&\cup_{\alpha=1}^{10}\{(i, j) ~\mid~ i, j \in I_\alpha, i\neq j\}.
\end{aligned}
\end{equation}

First, we compute the values $\gamma_{\max}$ and $\gamma_{\min}$ defined by (\ref{equ: gamma min &max}) and their ratio $r=\frac{\gamma_{\max}}{\gamma_{\min}}$ on the original data $A^{\prime}$. The values are $$\gamma_{\min}= 0.0093,\quad\gamma_{\max}= 0.0887,\quad r=9.5397,$$
which implies that the model (\ref{model: WSON}) with above weights $w_{ij}$ can perfectly recover the true cluster membership of $A^{\prime}$ for any $\gamma\in[0.0093,0.0887)$. The large ratio $r$ implies the feasibility of the model (\ref{model: PSON}) with some suitable $\epsilon\in(0,1)$ under the same weights $w_{ij}$.

Next, we will calculate the theoretically valid embedding dimensions for both cases, respectively. In order to achieve this goal, we will calculate the values $\epsilon_{\min}$ and $\epsilon_{\sup}$ defined in Theorem \ref{thm_PSON2} and the values $\tilde{\epsilon}_{\min}$ and $\tilde{\epsilon}_{\sup}$ defined in Theorem \ref{thm_PSON3}, respectively.

If we take the embedding dimension as $m = O(\epsilon^{-2}\log(n)) = ceil(10\epsilon^{-2}\log(10000))$. The values $\epsilon_{\min}$ and $\epsilon_{\sup}$ defined in Theorem \ref{thm_PSON2} on the data $A'$ are $\epsilon_{\min}=0.9597$ and $\epsilon_{\max}=0.9782$. This implies that for $\epsilon\in[0.9597,0.9782)$ and ${\gamma} \in\left[\sqrt{1+\epsilon}\gamma_{\min}, \sqrt{1-\epsilon}\gamma_{\max}\right)$, the model (\ref{model: PSON}) with the corresponding embedding dimension $m$ can perform the perfect clustering recovery on $A^{\prime}$ with high probability. The lowest possible dimension reduction ratio $\epsilon_{\min}$ is very close to $1$, which implies that we can hardly obtain a sufficient dimension reduction effect by Theorem (\ref{thm_PSON2}). In fact, the lowest possible embedding dimension guaranteed by Theorem (\ref{thm_PSON2}) is $m=ceil(10\epsilon_{\sup}^{-2}\log(10000))=97$. We will choose a valid distortion $\epsilon= 0.975\in[0.9597,0.9782)$. and test with $m=97$. We will compute the theoretically estimated interval $\left[\sqrt{1+\epsilon}\gamma_{\min},\sqrt{1-\epsilon}\gamma_{\max}\right)$ in Theorem \ref{thm_PSON2} for perfect recovery with $\epsilon= 0.975$. Then, we will randomly sample a $\Pi$ and test whether the model (\ref{model: PSON}) could perform the perfect clustering recovery for $\gamma$ in the estimated interval. Results are listed in Table \ref{tab: log(n)}.

Now, we move on to consider taking $\tilde{m}= O(\epsilon^{-2}\log(K)) =  ceil(10\epsilon^{-2}\log(10))$. For a random matrix $\Pi \in \mathbb{R}^{\tilde{m} \times d}$ defined as (\ref{equ: Pi Gaussians}), it follows from Theorem II.13 in \citep{davidson2001local} and Theorem 2.6 in \citep{rudelson2010non} that, the two-side bounds $\bar{S}(\tilde{m},d,t)$ and $\underbar{S}(\tilde{m},d,t)$ in  (\ref{equ: subgaussian singular value estimation}) are
    $$
    \bar{S}(\tilde{m},d,t)=\frac{\sqrt{100} + t}{\sqrt{\tilde{m}}} + 1=\frac{10 + t}{\sqrt{\tilde{m}}} + 1,\quad \underbar{S}(\tilde{m},d,t)=\frac{\sqrt{100}- t}{\sqrt{\tilde{m}}}-1= \frac{10 - t}{\sqrt{\tilde{m}}} - 1.
    $$
   By setting $t=2$, with a probability over $1-2\exp(-2^2)=0.9634$, we have
    $$
    \begin{array}{c}
    s_{1}(\Pi)\leq\bar{S}(\tilde{m},100,2)=\frac{12}{\sqrt{\tilde{m}}} + 1,\\
    s_{\tilde{m}}(\Pi)\geq\underbar{S}(\tilde{m},100,2)=\frac{8}{\sqrt{\tilde{m}}} -1,
    \end{array}
    $$
    and the values $\tilde{\epsilon}_{\min}$ and $\tilde{\epsilon}_{\sup}$ defined in Theorem \ref{thm_PSON3} are then estimated to be $\tilde{\epsilon}_{\min}=0.4799$ and $\tilde{\epsilon}_{\max}=0.8863$. This implies that for any $ \epsilon\in[0.4799,0.8863)$,
    and ${\gamma} \in\left[\bar{S}(\tilde{m},100,2)\gamma_{\min},\sqrt{1-\epsilon}\gamma_{\max }\right)$, the model (\ref{model: PSON}) with embedding dimension $\tilde{m}$ can perform the perfect clustering recovery of the data $A^{\prime}$ with high probability. We choose $\epsilon\in\{0.70,0.85\}$ in the valid interval $[0.4799,0.8863)$. In other words, we will test with $\tilde{m}\in\{47,32\}$. For each $\tilde{m}$, we will first randomly sample 1000 independent $\Pi$, and then test the successful probability $p_{X_{\mathcal{C}(A)}^{\prime}}$ of the squared-norm preservation of the points in the set $X_{\mathcal{C}(A)}^{\prime}$ within the desired distortion, as well as the successful probability $p_{S}$ of the two-side bounds of extreme singulars of the random projection matrices. We will then compute the estimated range $\left[\bar{S}(\tilde{m},100,2)\gamma_{\min},\sqrt{1-\epsilon}\gamma_{\max }\right)$ in Theorem \ref{thm_PSON3} for perfect recovery. Finally, we will randomly sample a random projection matrix $\Pi$ for each $\tilde{m}$ and test test whether the model (\ref{model: PSON}) could do perfect recovery with $\gamma\in\left[\bar{S}(\tilde{m},100,2)\gamma_{\min},\sqrt{1-\epsilon}\gamma_{\max }\right)$. Results are listed in Table \ref{tab: log(K)}.

\begin{table}[tbhp]
\caption{The numerical performance of the model (\ref{model: PSON})  with embedding dimension $m = O(\epsilon^{-2}\log(n))$.}
\label{tab: log(n)}
{\footnotesize
\begin{center}
\begin{tabular}{|c|c|c|c|} \hline
 \bf{Dimension (distortion)} &$\left[\sqrt{1+\epsilon}\gamma_{\min},\sqrt{1-\epsilon}\gamma_{\max}\right)$ &\bf{Perfect recovery}\\ \hline
$m=97\quad (\epsilon= 0.975$) &
$[0.0131,0.0140)$ &  $\checkmark$\\\hline
\end{tabular}
\end{center}
}
\end{table}

\begin{table}[tbhp]
\caption{The numerical performance of the model (\ref{model: PSON}) with embedding dimension $\tilde{m} = O(\epsilon^{-2}\log(K))$.}
\label{tab: log(K)}
{\footnotesize
\begin{center}
\begin{tabular}{|c|c|c|c|c|c|} \hline
 \bf{Dimension (distortion)} &  $p_{X_{\mathcal{C}(A)}^{\prime}}$ & $p_{S}$ & $\left[\bar{S}(\tilde{m},100,2)\gamma_{\min},\sqrt{1-\epsilon}\gamma_{\max }\right)$ &\bf{Perfect recovery}\\ \hline
$\tilde{m}=47\quad (\epsilon= 0.70$) & 921/1000 & 1000/1000 &$[0.0256,0.0486)$ & $\checkmark$\\\hline
$\tilde{m}=32\quad (\epsilon= 0.85$) & 915/1000& 1000/1000 & $[0.0290,0.0344)$ & $\checkmark$\\\hline
\end{tabular}
\end{center}
}
\end{table}

From the results in Table \ref{tab: log(n)} and Table \ref{tab: log(K)}, we may observe that, under the settings in this section,  we can only reduce the original dimension $d=100$ to $m=97$ theoretically if we take $m = O(\epsilon^{-2}\log(n))$. In contrast, if we take $m = O(\epsilon^{-2}\log(K))$, we can reduce the data dimension from $d=100$ to $m=32$. The above experiments demonstrate that the embedding dimension of the model (\ref{model: PSON}) can be further improved from $O(\epsilon^{-2}\log(n))$ to $O(\epsilon^{-2}\log(K))$.

\subsection{Robustness of the Randomly Projected Convex Clustering Model under Practical Settings}
\label{sec: numerical-sec2}
In this section, we will focus on further demonstrating the robustness of the model (\ref{model: PSON}) under practical settings. We will demonstrate from two perspectives: The robustness of different problem scales and embedding dimensions. First of all, we will conduct some analysis on the practical settings for (\ref{model: PSON}), in terms of weights $w_{ij}$ and the embedding dimension $m$. In terms of experiments, we will first exploit the potential of the model (\ref{model: PSON}) by choosing lower embedding dimensions on data $A$. Then, we will test on six more simulated balanced Gaussian data with different dimension $d$, size $n$, and ground-truth cluster number $K$. We will also provide numerical experiments on some unbalanced Gaussian data. The datasets are described in details later.

\subsubsection{Practical Settings of the Randomly Projected Convex Clustering Model}
Recall the settings we use in the numerical verification of the model (\ref{model: PSON}) on data $A$: 1. For weights $w_{i j}$, we choose the Gaussian kernel weights (\ref{equ: weight-setting}) with a well-designed graph (\ref{equ: weight graph for A}). 2. For the embedding dimension $m$, we set $m=ceil(9\epsilon^{-2}\log(n))$, where $\epsilon\in(0,1)$ is some desired distortion. These settings guarantee the conditions in the recovery guarantee of the model (\ref{model: PSON}): (1) $w_{i j}>0$ and $n_\alpha w_{i j}>\mu_{i j}^{(\alpha)}$ for all $i, j \in I_\alpha, \alpha\in [K]$. (2) With high probability, a random projection matrix $\Pi$ could preserve the squared norm for all the points in $X_{\mathcal{V}}$ and the centroids $X_{\mathcal{C}}$ within the desired distortion $\epsilon$.

In practical implementations of the model (\ref{model: PSON}), there are two challenges: First, we have no idea about the true cluster assignments of data. Second, computational efficiency should be taken into consideration. To overcome these challenges, we explore some robust and efficient practical settings. For the weights $w_{i j}$, since the Gaussian kernel weights (\ref{equ: weight-setting}) with a $k$-nearest neighbors graph has already demonstrated its robustness in the past literature \citep{chi2015splitting,yuan2018efficient,sun2021convex}, we simply choose the weights with a 10-nearest neighbors graph by default. We will focus more on testing the robustness regarding the embedding dimension $m$.

Although the mentioned two conditions for recovery guarantees might no longer hold in practical settings, our experimental results show that the practical performance of the model (\ref{model: PSON}) could still be robust. This motivates us to explore tighter and more general recovery guarantees of the model (\ref{model: PSON}) in further work.

\subsubsection{Robustness of the Randomly Projected Convex Clustering Model with Lower Embedding Dimensions}
We will test the robustness of the model (\ref{model: PSON}) on $A$ regarding $m$. We choose the same desired distortions $\epsilon\in\{0.2,0.4,0.6,0.8,0.95\}$ as in the previous section but set $m=ceil(\epsilon^{-2}\log(n))$. In other words, the corresponding embedding dimensions are $m\in\{173,44,20,11,8\}$, which are much lower than the previous setting with $m=ceil(9\epsilon^{-2}\log(n))$. For each $m$, we first randomly sample ten random projections $\Pi$. Then, we compute the averaged percentage of the squared norm of points that are successfully jointly preserved within the desired distortion $\epsilon$ in $X_A$, $X_{\mathcal{V}}$, and $X_{\mathcal{C}(A)}$. The results are listed in Table \ref{tab: check for pwd2}. We can observe from the results that over $93\%$ of points on average could still be preserved jointly within the desired distortion $\epsilon$. Next, we test the practical clustering performance of the model (\ref{model: PSON}) regarding all the ten randomly sampled projection matrices $\Pi$ on a clustering path generated by $\gamma=[10:-0.2:2]$. The results are summarized in Table \ref{tab: check for pwd2}. From Table \ref{tab: check for pwd2}, we can see that for each $m$, the model (\ref{model: PSON}) can perform perfect recovery robustly for all the ten randomly sampled projection matrices. These results show that the practical performance of the randomly projected convex clustering model is very robust.
\begin{table}[tbhp]
\caption{Averaged percentage of points in $X_A$, $X_{\mathcal{V}}$, and $X_{\mathcal{C}(A)}$ that the square-norm of these points can be jointly preserved by one random $\Pi$ within the desired distortion, and the recovery results.}
\label{tab: check for pwd2}
{\footnotesize
\begin{center}
\begin{tabular}{|c|c|c|c|c|} \hline
 \bf{Dimension (distortion)} & $X_A \%$ & $X_{\mathcal{V}}\%$& $X_{\mathcal{C}(A)}\%$ & \bf Perfect recovery \\ \hline
$m=173\quad (\epsilon= 0.2$) &  93.70\% &  93.72\% & 93.12\% &10/10
\\\hline
$m=44\quad (\epsilon= 0.4$) &  94.32\% & 94.44\% & 94.79\%&10/10
\\\hline
$m=20\quad (\epsilon= 0.6$) &  94.78\% & 94.85\% & 95.38\%&10/10
\\\hline
$m=11\quad (\epsilon= 0.8$) &  95.04\% & 95.09\% & 95.07\% &10/10
\\\hline
$m=8\quad (\epsilon= 0.95$) &  95.16\% & 95.26\% & 95.17\%&10/10
\\\hline
\end{tabular}
\end{center}
}
\end{table}
\subsubsection{Robustness of the Randomly Projected Convex Clustering Model with Different Problem Scales}
We will test the robustness of the model (\ref{model: PSON}) with different problem scales. We first test on balanced Gaussian data of different scales. In particular, we choose the scale $(d,n,K)\in\{(10^{2},10^{3},10), (10^{3},10^{3},10), (10^{4},10^{3},10), (10^{3},10^{3},2), (10^{3},10^{3},50), (10^{4},10^{4},50)\}$, and we set $\boldsymbol{\mu}_{k}=\mathbf{e}_k$,
$\sigma_{k}^{2}=0.005$, and $w_k=\frac{1}{K}$, for $k=1,\ldots,K$. The above six data sets are visualized in Figure \ref{Fig:view_balancedMSG}.
\begin{figure}[tbhp]
\centering
\subfloat[$(100,1000,10)$]{\label{fig: MSGd100n1000K10sigma}\includegraphics[width=0.33\textwidth]{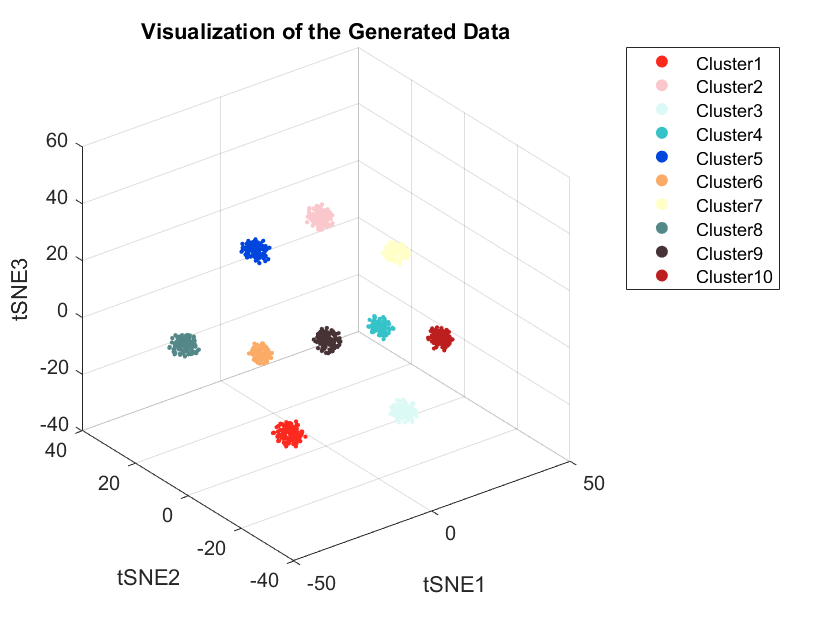}}
\subfloat[ $(1000,1000,10)$]{\label{fig: MSGd1000n1000K10sigma}\includegraphics[width=0.33\textwidth]{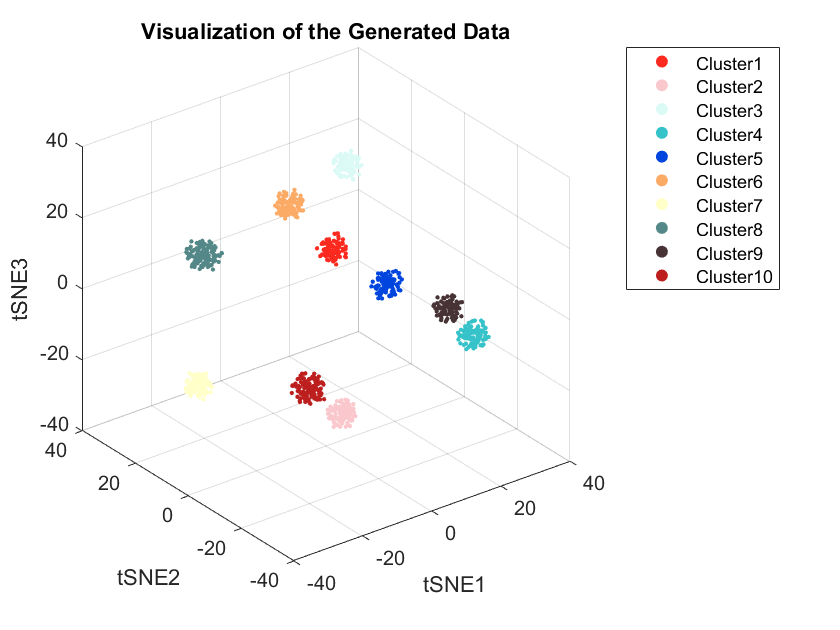}}
\subfloat[ $(10000,1000,10)$]{\label{fig: MSGd10000n1000K10sigma}\includegraphics[width=0.33\textwidth]{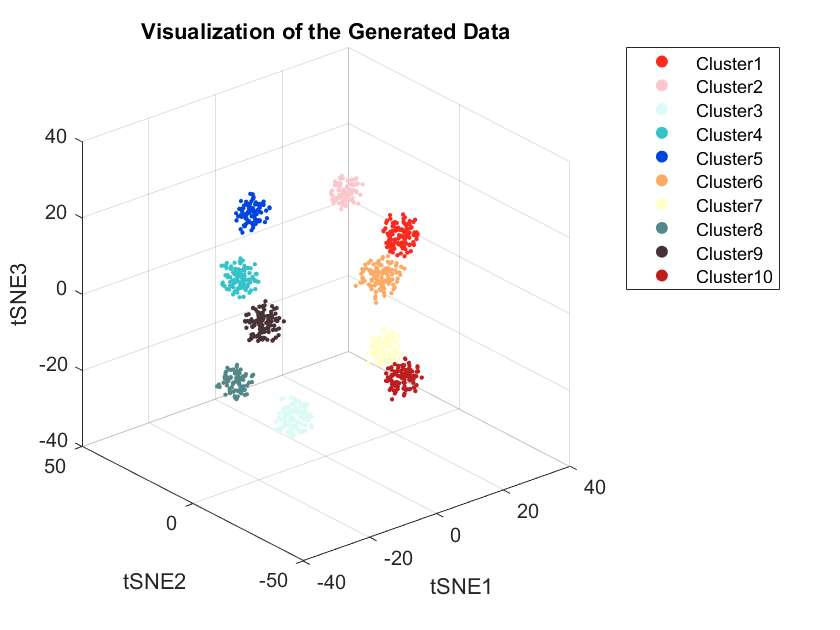}}

\subfloat[ $(1000,1000,2)$]{\label{fig: MSGd1000n1000K20sigma}\includegraphics[width=0.33\textwidth]{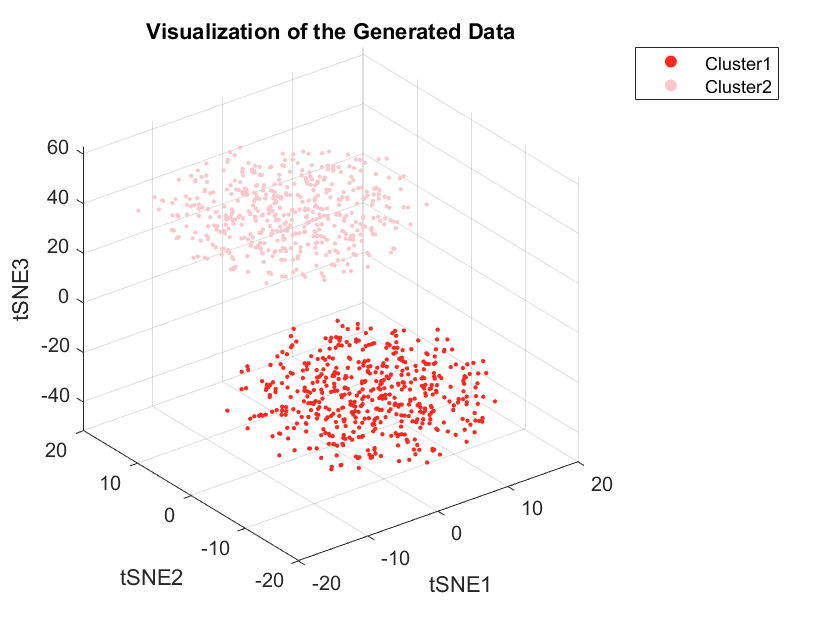}}
\subfloat[ $(1000,1000,50)$]{\label{fig: MSGd1000n1000K50sigma}\includegraphics[width=0.33\textwidth]{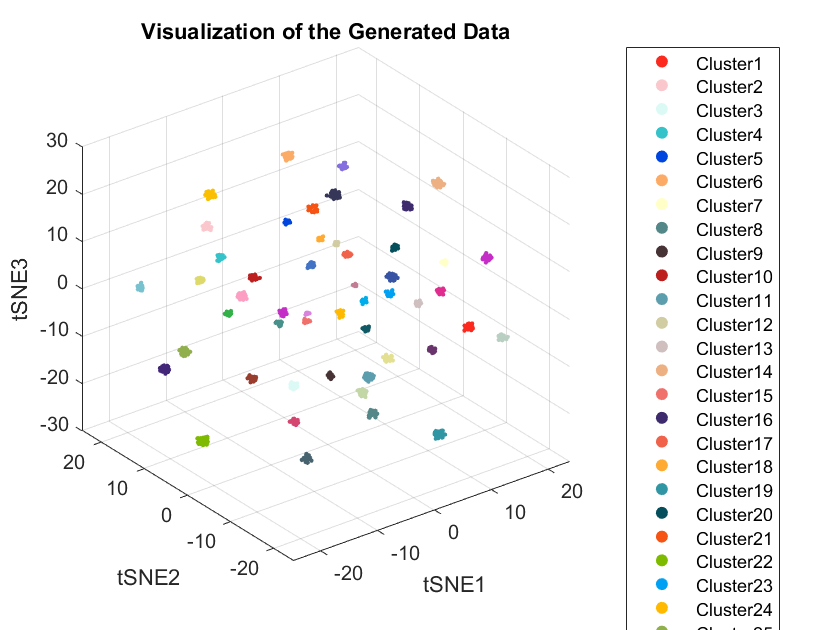}}
\subfloat[$(10000,10000,50)$]{\label{fig: MSGd10000n10000K50sigma}\includegraphics[width=0.33\textwidth]{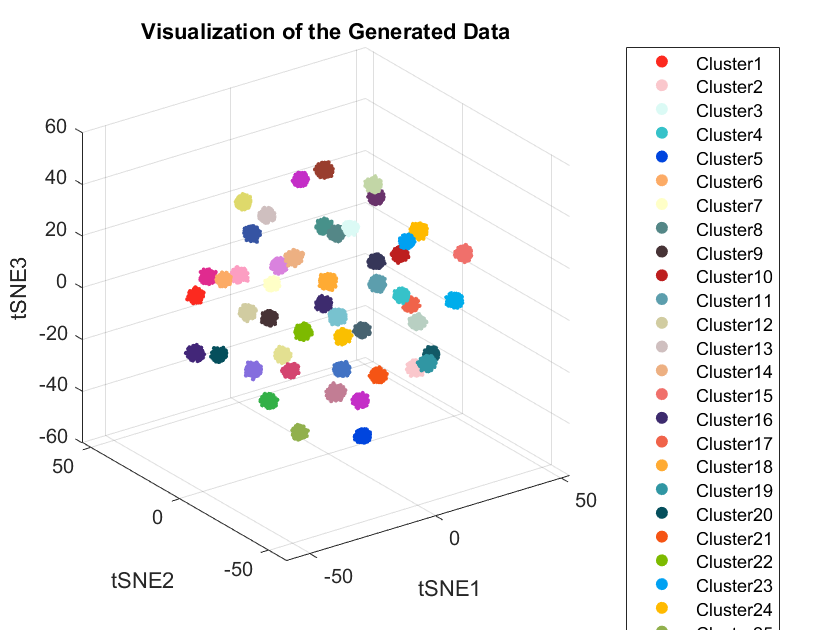}}
\caption{Visualization for six balanced Gaussian data of scale $(d,n,k)$.}
\label{Fig:view_balancedMSG}
\end{figure}

For each data, we randomly sample ten random projection matrices $\Pi\in\mathbb{R}^{m\times d}$ for every $m=10,20,50$. The clustering performance of the model (\ref{model: PSON}) along $\gamma=[10:-0.2:2]$ is summarized in Table \ref{tab: clustering performance on balanced Gaussian}. The results show that the model (\ref{model: PSON}) is robust to the scale of the data in practice.
\begin{table}[tbhp]
\caption{Clustering performance of RPCCM with  $m=10,20,50$ along $\gamma=[10:-0.2:2]$ on six balanced Gaussian data.}
\label{tab: clustering performance on balanced Gaussian}
{\footnotesize
\begin{center}
\begin{tabular}{|c|c|} \hline
 \bf{Dimension} & \bf {Perfect recovery}
 \\ \hline
$m=50$ &60/60
 \\ \hline
$m=20$ &60/60
\\\hline
$m=10$  &60/60
\\\hline
\end{tabular}
\end{center}
}
\end{table}

We also test on an unbalanced Gaussian data generated from $20$ spherical Gaussians $\mathcal{N}(\mathbf{e}_k,0.005I_{d})$ for all $k=1,\ldots, 20$, containing $7700$ samples in total. In particular, there are $2000$ samples for each of the first three clusters, and there are $100$ samples for each of the rest $17$ clusters. Again, for each $m=10,20,50$, we randomly sample ten projection matrices $\Pi\in\mathbb{R}^{m\times d}$. We then compare the clustering performance of the model (\ref{model: PSON}) and the RP $20$-means model. We generate the clustering path of the model (\ref{model: PSON}) with $\gamma=[10:-0.2:2]$. The results are summarized in Table \ref{tab: clustering performance on Ubalanced Gaussian}, which demonstrate the effectiveness and robustness of the model (\ref{model: PSON}).
\begin{figure}[tbhp]
\centering
\includegraphics[width=0.5\textwidth]{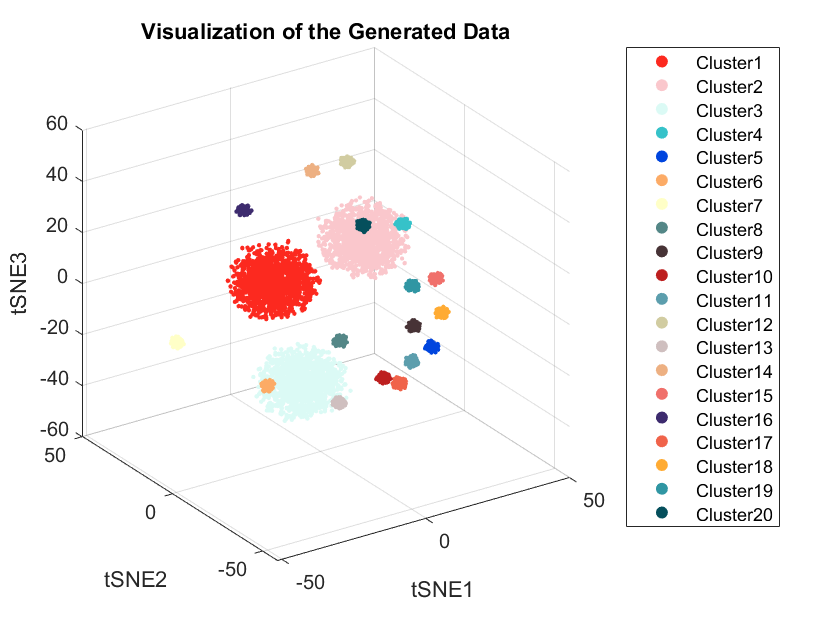}
\caption{Visualization for the unbalanced Gaussian data.}
\label{Fig:view_UbalancedMSG}
\end{figure}

\begin{table}[tbhp]
\caption{Clustering performance of RPCCM and RP $20$-means on the unbalanced Gaussian data (in terms of averaged rand index and adjusted rand index). }
\label{tab: clustering performance on Ubalanced Gaussian}
{\footnotesize
\begin{center}
\begin{tabular}{|c|c|c|c|} \hline
 \bf{Clustering model} & \bf {Perfect recovery} & \bf {rand index} & \bf {adjusted rand index}
 \\ \hline
RPCCM ($m=50$) &10/10 &\bf{1.0000}&\bf{1.0000}
 \\ \hline
RPCCM ($m=20$) &10/10&\bf{1.0000}&\bf{1.0000}
\\\hline
RPCCM ($m=10$)  &10/10 &\bf{1.0000}&\bf{1.0000}
\\\hline
RP $20$-means ($m=50$) &0/10 &0.8211&0.2343
 \\ \hline
 RP $20$-means ($m=20$) &0/10 & 0.7899&0.1061
 \\ \hline
 RP $20$-means ($m=10$) &0/10 &0.7771&0.0503
 \\ \hline
\end{tabular}
\end{center}
}
\end{table}

\subsection{Practical Performance of the Randomly Projected Convex Clustering Model on the Real Data}
\label{sec: numerical-true}
In this section, we will test the practical performance of the model (\ref{model: PSON}) on the lung cancer data \citep{lee2010biclustering}. The lung cancer data contains the microarray gene expressions of $12625$ genes for $56$ subjects belonging to one of four disease subgroups: Normal subjects (Normal), pulmonary carcinoid tumors (Carcinoid), colon metastases (Colon), and small cell carcinoma (Small Cell). In the models (\ref{model: WSON}) and (\ref{model: PSON}), we will compute the weights $w_{ij}$ following (\ref{equ: weight-setting}) with a 5-nearest neighbors graph. For the embedding dimension of the model (\ref{model: PSON}), we will set $m\in\{10,20,100,500\}$. For each $m$, we will randomly sample a random projection matrix $\Pi \in \mathbb{R}^{m \times d}$ following (\ref{equ: Pi Gaussians}). We will then test the practical performance of the models (\ref{model: WSON}) and (\ref{model: PSON}) by generating a clustering path with $\gamma\in [1:1:35]\cup [36:20:556]$. We visualize the clustering paths in Figure \ref{Fig:view_lung}.

\begin{figure}[tbhp]
\centering
\subfloat[CCM ($d=12625$)]{\label{fig:lung_d}\includegraphics[width=0.33\textwidth]{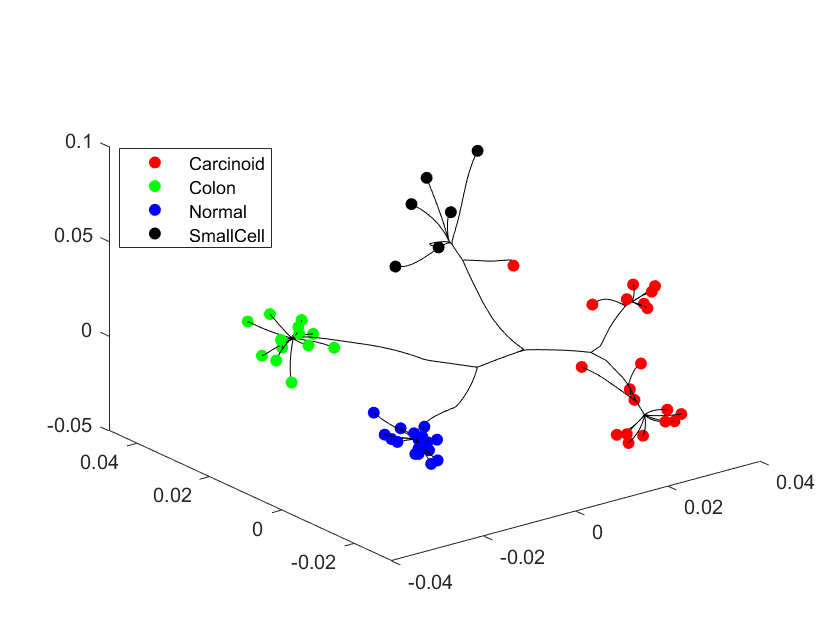}}
\subfloat[ RPCCM ($m=500$)]{\label{fig:lung_m500}\includegraphics[width=0.33\textwidth]{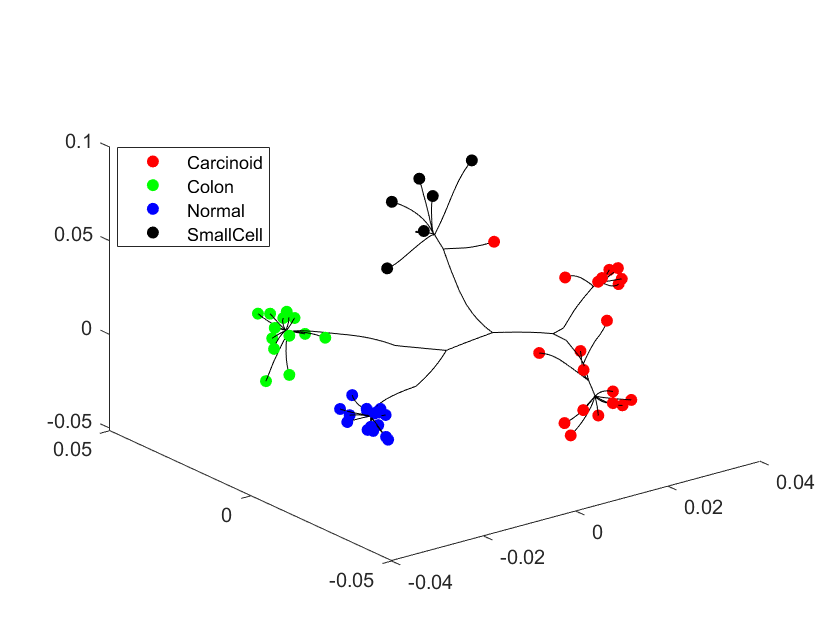}}
\subfloat[RPCCM ($m=100$)]{\label{fig:lung_m100}\includegraphics[width=0.33\textwidth]{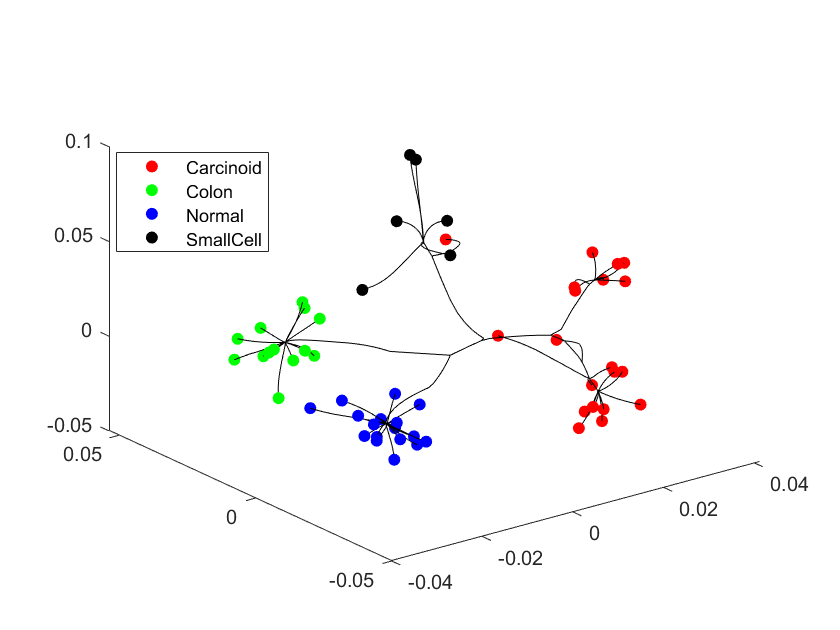}}

\subfloat[RPCCM ($m=20$)]{\label{fig:lung_m20}\includegraphics[width=0.33\textwidth]{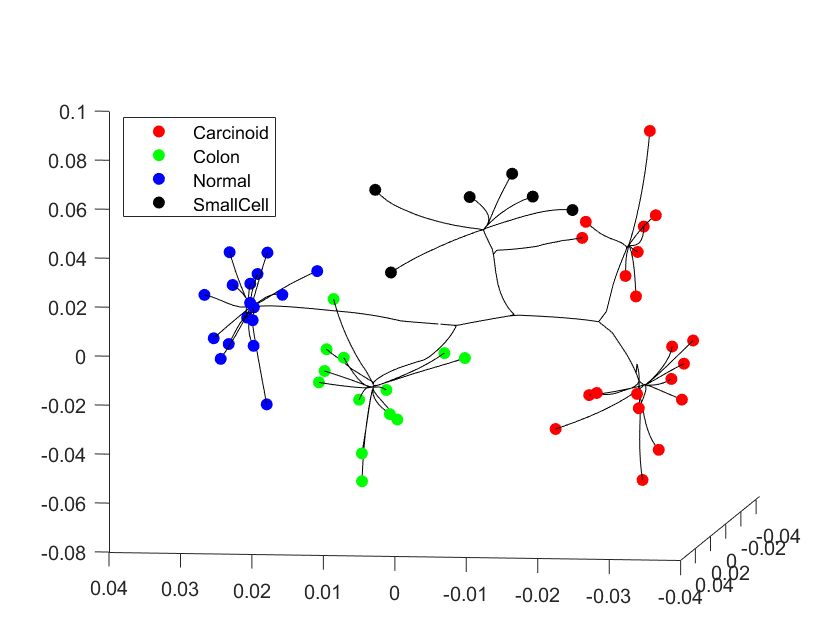}}
\subfloat[RPCCM ($m=10$)]{\label{fig:lung_m10}\includegraphics[width=0.33\textwidth]{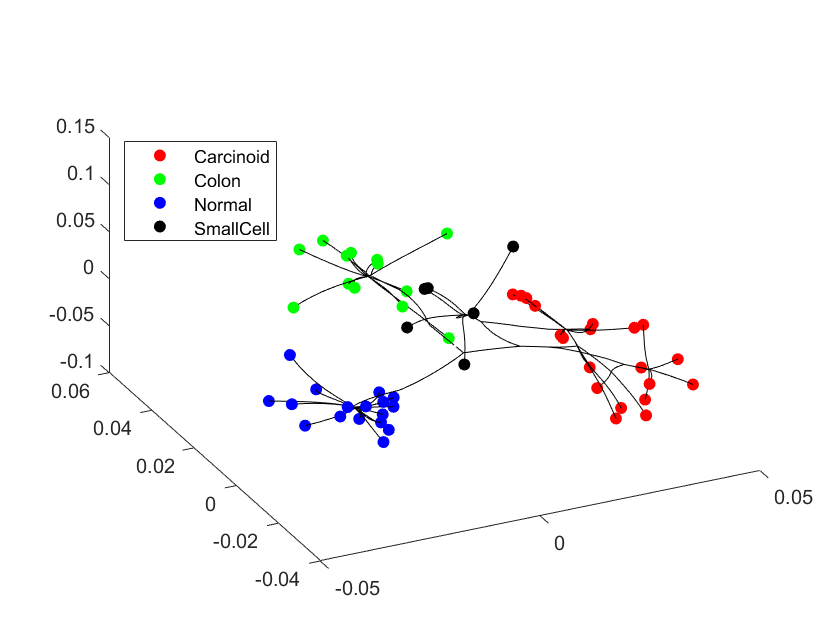}}

\caption{Visualization of the clustering paths.}
\label{Fig:view_lung}
\end{figure}

From the visualizations, we can observe that the convex clustering model (\ref{model: WSON}) performs well on this real data set, where only one data point from the Carcinoid cluster is clustered wrongly. A possible reason is that this wrongly clustered data point is closer to the SmallCell cluster. Moreover, the superior performance of the convex clustering model can be properly preserved by the model (\ref{model: PSON}), even for a very low embedding dimension. More detailed numbers can be found in Table \ref{tab: clustering performance on lung data}.

We also compare the clustering performance of the model (\ref{model: PSON}) with the RP K-means model. For the sake of fairness, we will test with the true number of clusters $K=4$. The results are summarized in Table\ref{tab: clustering performance on lung data}. The results show that the performance of the RP K-means model becomes less reliable as $m$ decreases, while the model (\ref{model: PSON}) is robust.

\begin{table}[tbhp]
\caption{Clustering performance of RPCCM and RP 4-means on the lung cancer data. Here, accuracy means the ratio of correctly clustered data points, and $\gamma^*$ is some value of $\gamma$ corresponding to the best clustering results of RPCCM.}
\label{tab: clustering performance on lung data}
{\footnotesize
\begin{center}
\begin{tabular}{|c|c|c|c|c|} \hline
 \bf{Clustering model} & \bf {accuracy} & \bf {rand index} & \bf {adjusted rand index} & \bf $\gamma^*$
 \\ \hline

RPCCM ($d=12625$) &55/56 & 0.9838 & 0.9586  & $76$
 \\ \hline
RPCCM ($m=500$) &55/56 & 0.9838 & 0.9586  & $76$
 \\ \hline
RPCCM ($m=100$) &55/56 & 0.9838 & 0.9586  & $76$
 \\ \hline
RPCCM ($m=20$) &55/56 & 0.9838 & 0.9586  & $76$
 \\ \hline
RPCCM ($m=10$) &55/56 & 0.9838 & 0.9586  & $96$
 \\ \hline
4-means ($d=12625$) &55/56 & 0.9838 & 0.9586  & /
 \\ \hline
RP 4-means ($m=500$) &55/56 & 0.9838 & 0.9586  & /
 \\ \hline
RP 4-means ($m=100$)&54/56 & 0.9701 & 0.9245 & /
 \\ \hline
RP 4-means ($m=20$) &48/56 & 0.9000 & 0.7421 & /
 \\ \hline
RP 4-means ($m=10$) &43/56 &  0.8753 & 0.6795  & /
 \\ \hline
\end{tabular}
\end{center}
}
\end{table}

\section{Conclusion and Future Works}
\label{sec: conclusion}
In this paper, we proposed a randomly projected convex clustering model for clustering high dimensional data. We proved that, under some mild conditions, the perfect recovery of the cluster membership assignments of the convex clustering model on the original data, if exists, can be preserved by the randomly projected convex clustering model with a much smaller embedding dimension. In particular, we proved that the embedding dimension can be $m = O(\epsilon^{-2}\log(n))$, where $n$ is the number of data points and $0 < \epsilon < 1$ is some given tolerance. We further proved that the embedding dimension can be $m = O(\epsilon^{-2}\log K)$, where $K$ is the number of hidden clusters, which is independent of the number of data points. Extensive numerical experiment results were presented in this paper to demonstrate the robustness and superior performance of the randomly projected convex clustering model. The numerical results presented in this paper also demonstrated that the randomly projected convex clustering model can outperform the randomly projected K-means model in practice.

It is worthwhile pointing out that the practical performance of the convex clustering model and the randomly projected convex clustering model depends on the quality of the input data features. We regard it as a future research direction to investigate a new technique that can do dimension reduction and feature representation learning simultaneously.

\acks{The research of Yancheng Yuan is supported in part by The Hong Kong Polytechnic University under  grant P0038284. The research of Defeng Sun is supported in part by the Hong Kong Research Grant Council under grant 15304721.}


\end{document}